\documentclass[twoside]{article}
\usepackage{multicol}
\usepackage{setspace}
\usepackage[utf8]{inputenc} 
\usepackage[T1]{fontenc}    
\usepackage{hyperref}       
\usepackage{url}            
\usepackage{booktabs}       
\usepackage{amsfonts}       
\usepackage{nicefrac}       
\usepackage{microtype}      
\usepackage{xcolor}         
\usepackage[linesnumbered,ruled,vlined]{algorithm2e}
\usepackage{algpseudocode}
\usepackage{amsmath,amsthm,amsfonts,amssymb,mathtools,bbm}
\usepackage{color}
\usepackage{xcolor}
\usepackage{mathrsfs}
\usepackage{enumitem} 
\usepackage{bm}
\usepackage{graphicx}
\usepackage{dsfont}
\usepackage[labelfont=bf]{subcaption} 
\usepackage{blindtext}
\usepackage{capt-of}
\usepackage{booktabs}
\usepackage{varwidth}
\usepackage{adjustbox}
\usepackage{lipsum} 
\usepackage{cuted}  
\usepackage{caption} 
\usepackage{pgfplots}
\pgfplotsset{compat=1.18}
\usepgfplotslibrary{external}
\newsavebox\tmpbox

\newtheorem{theorem}{Theorem}[section]

\newtheorem{lemma}[theorem]{Lemma}
\newtheorem{definition}[theorem]{Definition}
\newtheorem{proposition}[theorem]{Proposition}
\newtheorem{corollary}[theorem]{Corollary}

\usepackage[accepted]{aistats2026}
\usepackage[round]{natbib}

\begin{document}

%
\runningtitle{Adaptive Combinatorial Experimental Design}

%
\runningauthor{H. Xie, J. Cao, K. Xu}

\twocolumn[

\aistatstitle{Adaptive Combinatorial Experimental Design: \\
Pareto Optimality for Decision-Making and Inference}

\aistatsauthor{%
Hongrui Xie\\
University of Science and \\
Technology of China\\
\And 
Junyu Cao\\
The University of Texas at Austin\\
\And 
Kan Xu\\
Arizona State University\\
}

\aistatsaddress{}
]

\begin{abstract}
In this paper, we provide the first investigation into adaptive combinatorial experimental design, focusing on the trade-off between regret minimization and statistical power in combinatorial multi-armed bandits (CMAB). While minimizing regret requires repeated exploitation of high-reward arms, accurate inference on reward gaps requires sufficient exploration of suboptimal actions. We formalize this trade-off through the concept of Pareto optimality and establish equivalent conditions for Pareto-efficient learning in CMAB. We consider two relevant cases under different information structures, i.e., full-bandit feedback and semi-bandit feedback, and propose two algorithms \texttt{MixCombKL} and \texttt{MixCombUCB} respectively for these two cases. We provide theoretical guarantees showing that both algorithms are Pareto optimal, achieving finite-time guarantees on both regret and estimation error of arm gaps. Our results further reveal that richer feedback significantly tightens the attainable Pareto frontier, with the primary gains arising from improved estimation accuracy under our proposed methods. Taken together, these findings establish a principled framework for adaptive combinatorial experimentation in multi-objective decision-making.
\end{abstract}

\section{Introduction}
Combinatorial multi-armed bandits (CMAB) generalize the classical multi-armed bandit (MAB) framework to settings where the learner selects a structured combination of basic actions, referred to as a ``super arm,'' at each round. This setup models many real-world problems, such as online advertising, sensor selection, network routing, and recommendation systems \citep{CESABIANCHI20121404,pmlr-v28-chen13a}, where multiple actions are taken jointly and rewards depend on the combination, and many of them require accurate estimation of reward gaps between combinatorial arms as well as regret minimization. For example, empirical evidence from a major video-sharing platform shows that joint treatment effects can be strongly non-additive \citep{Ye_Zhang_Zhang_Zhang_Zhang_2023}. This setting is inherently combinatorial: at each round the platform selects a set of interventions (a super-arm) and observes only session-level outcomes (e.g., total watch time), aligning with the regret–inference trade-offs studied in our combinatorial bandit models. 
In such combinatorial problems, this tension is exacerbated by the large action space and dependencies between arms. The combinatorial nature of the action space increases the complexity of both exploration and optimization. We investigate the problem of deriving Pareto optimal policies for combinatorial bandits, where the learner must jointly minimize regret and accurately estimate reward gaps between combinatorial arms. A policy is Pareto optimal if no alternative policy can be better off in both regret and estimation error. 
Our goal is to identify a class of Pareto optimal algorithms that strike proper balance between regret control and inference accuracy. 

To the best of our knowledge, our work provides the first systematic study of Pareto optimality in the context of combinatorial bandit feedback. We summarize our main contributions as follows:
\begin{enumerate}
\item \textbf{Pareto-Optimal Algorithms for Combinatorial Bandits.} We develop two Pareto optimal algorithms for two different combinatorial bandit settings -- i.e., \textsc{\texttt{MixCombKL}} for full-bandit feedback and \textsc{\texttt{MixCombUCB}} for semi-bandit feedback. Our algorithms dynamically calibrate exploration to preserve both estimation accuracy and regret performance to maintain Pareto optimality, regardless of the complexity of combinatorial action spaces. 
\item \textbf{Theoretical Guarantees under Different Bandit Feedback Models.} We derive finite-sample regret bounds and estimation error guarantees for both bandit settings, and establish their Pareto optimality. Our analysis shows that semi-bandit feedback yields a sharper Pareto frontier than full-bandit feedback, with the improvement arising from reduced estimation error. Meanwhile, the regret under the two feedback regimes remains of the same order when using proposed algorithms.
\end{enumerate}

\section{Related Literature}

Our paper is related to the stochastic multi-armed bandit (MAB) literature of regret minimization and best-arm identification (BAI).
Regret minimization seeks to reduce cumulative loss \citep{MAL-068,kuleshov2014algorithms}, while BAI targets high-confidence identification of the optimal arm under fixed-budget
or fixed-confidence settings \citep{audibert:hal-00654404,pmlr-v49-garivier16a}. 
The combinatorial multi-armed bandit (CMAB) problem generalizes the classical bandit setting by allowing the learner to select a subset of base arms at each round. The literature distinguishes between \emph{semi-bandit} feedback, revealing individual rewards for basic actions \citep{pmlr-v38-kveton15,pmlr-v80-wang18a}, and \emph{full-bandit} feedback, revealing only total reward for super arms \citep{CESABIANCHI20121404,NIPS2015_0ce2ffd2}.
CMAB variants include contextual \citep{qin2014contextual,NEURIPS2018_207f8801,li2016contextual}, knapsack \citep{sankararaman2018combinatorial}, fairness \citep{li2019combinatorial}, relative-feedback settings \citep{NEURIPS2019_5e388103} and pure exploration \citep{NIPS2014_d3ea0f33}. 

The goals of regret minimization and arms inference are inherently misaligned: better inference requires more exploration, and thus may increase regret. Prior work focuses primarily on classical $K$-arm bandits. \cite{pmlr-v65-chen17b,10.5555/2946645.2946646,pmlr-v89-degenne19a} study this trade-off, showing no single algorithm achieves optimal rates for both simultaneously. Similarily, this tradeoff is further studied in \cite{faruk2025learningpeerinfluenceprobabilities} for linear contextual bandit settings. Pareto frontiers for MAB problem were first developed in \cite{zhong2023achievingparetofrontierregret}.
\cite{doi:10.1287/mnsc.2023.00492} then prove a sufficient and necessary condition of Pareto optimality and propose a Pareto optimal algorithm. \cite{zuo2025pareto} extend this notion to multinomial logistic bandit problems and \cite{zhang2025onlineexperimentaldesignestimationregret} extended this framework to network inference problems. However, to the best of our knowledge, the trade-off between regret minimization and arms inference remains unexplored for CMAB problems. 

Inspired from the Pareto optimality on classical MAB framework of \cite{doi:10.1287/mnsc.2023.00492}, our work characterized the regret-inference tradeoff through analyzing the Pareto optimality of CMAB problems. Applying the existing methods is computationally infeasible due to the exponential super-arm space in CMAB settings. Therefore, we design two different Pareto-optimal algorithms for bandit problems with different information structures, i.e., full-bandit and semi-bandit settings. Additionally, we show how feedback richness governs trade-offs, comparing the two regimes. We further overcome the challenge of CMAB problems on the concentration of information in super arm subspace, completing a unified Pareto-optimal framework.

\section{Problem Formulation}
\subsection{Models}

We consider the stochastic CMAB problem and define a bandit instance as a tuple $(\mathcal{A},\mathcal{M},\nu)$. In particular, $\mathcal{A} = \{1,2,\dots,d\}$ is the set of base arms, $\mathcal{M}\subseteq 2^{\mathcal{A}}$ is the set of super arms that contains the feasible subsets of $\mathcal{A}$, and $\nu$ is the distribution of the reward vector $w_t$ where $\mathbb{E}_{w_t\sim\nu}[w_t]=\boldsymbol{\mu} = [\mu(1),\dots, \mu(d)]^\top\in[0,1]^d$. $\mathcal{V}_0$ is the set of all admissible CMAB instances $\nu$. The basic actions in $\mathcal{A}$ are associated with a random reward vector $w_t$ at time $t$, drawn i.i.d. from a distribution $\nu$ over $[0,1]^d$, i.e., $w_t\sim\nu$. For each base arm $e\in\mathcal{A}$, its reward is the $e$-th coordinate denoted $w_t(e)\in[0,1]$. Define the mapping $f:2^{\mathcal{A}}\times\mathbb{R}^d\rightarrow\mathbb{R}, f(G,\varpi)=\sum_{e\in G}\varpi(e)$ as the total weight of the elements in an arbitrary set $G \subseteq \mathcal{A}$ under $d$-dimensional vector $\varpi$ with entries $\varpi(e)$; then, $f(M(t),w_t) = \sum_{e \in M(t)} w_t(e)$ denotes the reward at time $t$ for playing super arm $M(t)$. 

The time horizon is set to be $n$. At each round $t \leq n$, the decision maker observes the history $\mathcal{H}_t = (M(1); w_1, \dots, M(t); w_t)$, and selects a super arm $M(t)\subseteq\mathcal{M}$. An admissible policy $\pi = \{\pi_t\}_{t \geq 1}$ maps the history $\mathcal{H}_{t-1}$ to a super arm $M^{\pi}(t) \in \mathcal{M}$, where $\pi_t(M) = \mathbb{P}(M(t) = M \mid \mathcal{H}_{t-1})$ denotes the probability of selecting $M$ at time $t$. We define $M^* \coloneqq \arg\max_{M \in \mathcal{M}}\sum_{e\in M}\mu(e)$ to be the arm with the maximum true reward, and the performance of a policy $\pi$ is evaluated via cumulative \emph{regret}, which measures the expected loss relative to the optimal super arm $M^*$:
$\mathcal{R}_{\nu}(n, \pi) = \sum_{t=1}^n[\mathbb{E}[f(M^*,w_t)]-\mathbb{E}^\pi \left[f(M^{\pi}(t),w_t) \right]]$. 

We also want to quantify estimation errors for both base arms and super arms for any algorithm. We define the gap between any two super arms $M(\tau_i)$ and $M(\tau_j)$ as $\Delta_M^{(i,j)} \coloneqq f(M(\tau_i),\boldsymbol{\mu}) - f(M(\tau_j),\boldsymbol{\mu}),\forall i \ne j \in [|\mathcal{M}|]$; here $\tau_i,\tau_j$ are indexes of   super arms. In particular, the gap to the optimal super arm $M^*$ for every $M$ is $\Delta_{M} = f(M^{*},\boldsymbol{\mu}) - f(M,\boldsymbol{\mu})$. We also define the gap between base arms as $\Delta_\mu^{(i,j)} \coloneqq \mu(i) - \mu(j)$ for any $i \ne j \in [d]$. In addition, we assume each super arm satisfies the constraint $|M|= m$ for all $M \in \mathcal{M}$. For any algorithm that outputs estimation for bandit features, an algorithm-specific admissible adaptive estimator $\hat{\Delta}^{(i,j)} = \{\hat{\Delta}_t^{(i,j)}\}_{t \geq 1}$ maps $\mathcal{H}_t$ to an estimate of the pairwise gap $\Delta^{(i,j)}$ at time $t$ in our stochastic CMAB setting, where $\Delta^{(i,j)}$ can be either $\Delta^{(i,j)}_{M}$ or $\Delta^{(i,j)}_{\mu}$. The estimation quality is quantified by the expected distance of $\Delta^{(i,j)}$ and $\hat{\Delta}^{(i,j)}$ (i.e., $\mathcal{E}(t, \hat{\Delta}^{(i,j)}) = \mathbb{E}[|\Delta^{(i,j)} - \hat{\Delta}_t^{(i,j)}|]$ for $\Delta^{(i,j)}_t=\Delta^{(i,j)}_{M,t}$ or $\Delta^{(i,j)}_{\mu,t}$), which is treated as the \emph{estimation error}.

The richness of the bandit feedback influences the regret and estimation error and thus the Pareto frontier. To that end, we consider the two bandit feedback regimes upon selecting super arm $M(t)$ at time $t$: (i) \emph{full-bandit feedback}, where only the aggregate reward $f(M(t),w_t)$ is revealed with no information about the contributions of individual basic actions, and (ii) \emph{semi-bandit feedback}, where the individual component reward $w_t(e)$ is observed for $e\in M$ (when the respective basic action is chosen) and provides rich item-level information.


\subsection{Pareto Optimality}

To balance regret and estimation accuracy in combinatorial bandits, we begin by formalizing the notion of Pareto optimality. Because combinatorial feedback generally prevents estimation of all basic actions in $\mathcal{A}$ (see Appendix \ref{app:rst_bandit} for an example of such a restricted bandit structure), we define $\mathcal{A}_{ad}$ as the subset of actions in $\mathcal{A}$ that can be reliably estimated.

Consider a policy $(\pi, \widehat{\Delta})$, where $\pi$ denotes the learning algorithm and $\widehat{\Delta}$ is an estimator of the reward gap between combinatorial super arms. The policy is said to be \emph{Pareto optimal} if there is no alternative admissible policy $(\pi', \widehat{\Delta}')$ that performs at least as well in both cumulative regret and estimation error, and strictly better in at least one of them.
\begin{definition}[Pareto Optimality]\label{def: pareto}
$(\pi, \widehat{\Delta})$ is Pareto optimal if there does not exist $(\pi', \widehat{\Delta}')$ such that:
\[
\mathcal{R}_{\nu}(n, \pi') \preceq \mathcal{R}_{\nu}(n, \pi),  \max_{i\leq j} \mathcal{E}\left(\widehat{\Delta}'^{(i,j)}\right) \preceq \max_{i\leq j} \mathcal{E}\left(\widehat{\Delta}^{(i,j)}\right)
\]
with at least one inequality \footnote{In our paper, for any two positive functions $f(n)$ and $g(n)$, we write $f(n)\preceq g(n)$ if $\frac{f(n)}{g(n)}$ is bounded by a positive constant for all $n$; specifically, the inequality holds true if and only if there exist constants $C_1,C_2>0$ such that $C_1 \le \frac{f(n)}{g(n)} \le C_2$ for all $n$.} being strict.
\end{definition}

Here we denote $\mathcal{R}_{\nu}(n,\pi)$ as the cumulative regret under the reward distribution $\nu$.
The estimation error $\Delta$ in the combinatorial bandit setting will refer either to a base-arm gap, defined as
$\Delta_{\mu}^{(i,j)} = \mu(i) - \mu(j); \forall i \neq j \in [|\mathcal{A}_{ad}|]$,
or to a super-arm gap, defined as
$\Delta_M^{(i,j)} \coloneqq f\bigl(M(\tau_i),\boldsymbol{\mu}\bigr) - f\bigl(M(\tau_j),\boldsymbol{\mu}\bigr),
\forall i \neq j \in [|\mathcal{M}|]$.

The choice of ignoring constant or logarithmic factors is standard in bandit asymptotics as bandit asymptotics is about identifying the fundamental scaling of regret with problem parameters, and proof-dependent artifacts would not change the underlying rate of optimality. Similar factor choice is also seen in \cite{doi:10.1287/mnsc.2023.00492}.

We further define the Pareto frontier $\mathcal{P}_f$ to comprise all policies that are not strictly dominated in both regret and estimation accuracy:
\begin{definition}[Pareto Frontier] The Pareto Frontier $\mathcal{P}_f$ is defined as
\label{def:paretof}
\begin{multline*}
\mathcal{P}_f = \bigg\{ (\pi, \widehat{\Delta}) \;\bigg|\;
\nexists (\pi', \widehat{\Delta}') \text{ s.t. } \mathcal{R}_{\nu}(n, \pi') 
\preceq \mathcal{R}_{\nu}(n, \pi), \\
\max_{i \leq j} \mathcal{E}\!\left(\widehat{\Delta}'^{(i,j)}\right) 
\preceq \max_{i \leq j} \mathcal{E}\!\left(\widehat{\Delta}^{(i,j)}\right) \bigg\}.
\end{multline*}    
\end{definition}
Policies in $\mathcal{P}_f$ represent efficient trade-offs --- any policy outside this set can be improved upon in at least one dimension without sacrificing the other. The search for Pareto optimal solutions can be cast as a multi-objective optimization problem:
\[
\min_{(\pi, \widehat{\Delta})} \max_{\nu \in \mathcal{V}_0} \left( \mathcal{R}_{\nu}(n, \pi), \max_{i\leq j} \mathcal{E}\left(\widehat{\Delta}^{(i,j)}\right) \right).
\]
This formulation captures the dual learning objective of minimizing regret while maintaining accurate reward gap estimations under worst-case scenarios.

\section{Algorithms and Theory}

Because of the differences in information structure, our algorithmic design is tailored to the feedback model. In the full-bandit setting, the exponential size of the super-arm space renders classical UCB-style confidence construction impractical. Traditional UCB methods rely on per-arm play counts to form confidence intervals, but in this setting the rewards of individual base arms are not directly observed. Constructing valid confidence bounds requires projecting super-arm rewards into a high-dimensional linear space, which is statistically costly. KL-divergence-based methods bypass this issue by working on the probability simplex over super arms and using divergence constraints to guide exploration, enabling regret analysis without enumerating exponentially many actions.
In contrast, the semi-bandit feedback setting reveals feedback for each base arm, which makes per-action confidence intervals easy to compute and super-arm bounds derivable with only a single call to the optimization oracle. Here, applying KL-divergence methods adds computational overhead without improving statistical performance, while UCB-style algorithms already provide distribution-independent guarantees. Hence, our algorithm design adopts KL-divergence-based methods for full-bandit feedback and UCB-based approaches for semi-bandit, aligning algorithm design with feedback richness and computational feasibility. We explicitly analyze the computational efficiency of our algorithms in Appendix \ref{app:compres}.

\subsection{Full-Bandit Feedback}

Before introducing our algorithm, we first define some additional notation. We vectorize the super arm $M\in\mathcal{M}$ by setting $\boldsymbol{\theta}_M=[\mathbb{I}\{1\in M\},...,\mathbb{I}\{d\in M\}]^\top$, where $|M|=m$ implies that $\|\boldsymbol{\theta}_M\|_1=m$. The vectorization trick helps us construct matrix-based super-arm features for computation. Notice that $\max_{M \in \mathcal{M}} \boldsymbol{\theta}_M^\top X = \max_{\rho \in \mathrm{Co}(\boldsymbol{\theta})} \rho^\top X,$
where $\mathrm{Co}(\boldsymbol{\theta})$ denotes the convex hull of $\{\boldsymbol{\theta}_M:M\in\mathcal{M}\}$. Next, by dividing each vector in the vectorized set $\mathcal{M}$ by $m$, we embed it into the $d$-dimensional simplex and define $\mathcal{Q}$ as the corresponding scaled convex set. Here and throughout, the notation $\tilde{\mathcal{O}}(\cdot)$ hides factors polylogarithmic in the time horizon $T$ and the dimension parameters $m,d$.

Building on the Online Stochastic Mirror Descent (OSMD) \citep{14982149-b382-3e1b-8d96-b63be1f53b29} framework, we propose the \textsc{\texttt{MixCombKL}} algorithm (Algorithm~\ref{alg:mixcombkl}), which uses the Kullback–Leibler (KL) divergence as the Bregman divergence for projection onto $\mathcal{Q}$. The KL divergence between two distributions $p, q \in \mathcal{Q}$ is defined as $\mathrm{KL}(p, q) = \sum_{i=1}^d p(i) \log \frac{p(i)}{q(i)}$. The projection of $q$ onto a closed convex set $\Xi$ of distributions is given by $p^* = \arg \min_{p \in \Xi} \mathrm{KL}(p, q).$

We set parameters 
$C= \lambda_{\min}m^{-\frac{3}{2}},
\gamma = [\sqrt{m \log \rho_{\min}^{-1}}+ \sqrt{C \big(C m^{2} d + m\big) n}]^{-1} \sqrt{m \log \rho_{\min}^{-1}}, \eta = \gamma C,$
and we define the problem-dependent constants $\lambda_{\min}$ and $\rho_{\min}$ as follows: 
Let $\lambda_{\min}$ be the smallest nonzero eigenvalue of $\mathbb{E}[\boldsymbol{\theta}_M\boldsymbol{\theta}_M^\top]$, where $M$ is uniformly distributed over $\mathcal{M}.$ We define the exploration-inducing distribution $\rho ^0\in \mathcal{P} : \rho _e^0= \frac 1{m| \mathcal{M} | }\sum _{M\in \mathcal{M} }\mathbb{I}(e\in M)$, $\forall e\in \mathcal{A}$, and let $\rho_{\mathrm{min}}=\operatorname*{min}_{e\in\mathcal{A}}m\rho_{e}^{0}$; $\rho^{0}$ is the distribution over basic actions $\mathcal{A}$ induced by the uniform distribution over $\mathcal{M}.$ We view the \emph{estimable basic actions} under \emph{full-combinatorial feedback} as special super arms $M$ with $|M|=1$, so we can define the set of basic actions covered by the span of all $\boldsymbol{\theta}_M$ (written as $\mathrm{span}(\boldsymbol{\theta})$) as $\mathcal{M}_{KL} = \{M\in\mathcal{A}: \boldsymbol{\theta}_M \in \mathrm{span}(\boldsymbol{\theta}) \}.$

The \texttt{MixCombKL} algorithm proceeds as follows. First, at each round \(t\), compute a probability distribution \(p_{t-1}\) over the super arms by decomposing it onto the pre-determined distribution $q^{\prime}_{t-1}$, mixed through $q_{t-1}' = (1-\gamma)q_{t-1} + \gamma \rho^0$.
\begin{algorithm}[htb]
\caption{Mixture-Based Combinatorial KL-divergence Algorithm (\texttt{MixCombKL})}
\label{alg:mixcombkl}
\KwIn{$\alpha\in[0,\frac{1}{2}],\mathcal{M},n$;}
Initialization: Set $ q_{0}=\rho^{0},\:R_0(M(\tau_k))=0,\forall k\in[|\mathcal{M}|],R_0(M(l))=0,\forall l\in[|\mathcal{M}_{KL}|]$\;
\For{$t=1,...,n$}{
Let $q_{t-1}' = (1-\gamma)q_{t-1} + \gamma \rho^0$\;
Decompose distribution $p_{t-1}$ over $\mathcal{M}$ so that $\sum_{M\in\mathcal{M}}p_{t-1}(M)\boldsymbol{\theta}_M=mq_{t-1}^{\prime}$\;
Set random variable $U_t$ where $\mathbb{P}(U_t=0)=1-\frac{1}{2t^{\alpha}}$ and $\mathbb{P}(U_t=1)=\frac{1}{2t^{\alpha}}$\;
Select arm $M(t)$ with distribution $\forall M\in\mathcal{M}$, $\pi_t(M)=p_{t-1}(M)\mathbb{I}\{U_t=0\}+\mathbb{I}\{U_t=1\}/|\mathcal{M}|$\;
Observe reward $f(M(t),w_t)$\;
Let $\Sigma_{t-1}= \mathbb{E}[\boldsymbol{\theta}_M\boldsymbol{\theta}_M^\top]$, where $M$ has law $p_{t-1}$\;
Set $\tilde{w}_t(e)=f(M(t),w_t)\Sigma_{t-1}^{+}\boldsymbol{\theta}_M$, where $\Sigma_{t-1}^{+}$ is the pseudo-inverse of matrix $\Sigma_{t-1}$\;
Set $\tilde{q}_{t}(e)=\frac{q_{t-1}(i)\exp(\eta\tilde{w}_t(e))}{\sum_{j=1}^{d}q_{t-1}(j)\exp(\eta\tilde{w}_t(e))},\:\forall e\in\mathcal{A}$\;
Compute $q_t=\mathbb{I}\{U_t=1\}q_{t-1}+\mathbb{I}\{U_t=0\}$ $
\arg\min_{p\in\mathcal{Q}}\sum_{e\in\mathcal{A}}p(e)\log\frac{p(e)}{\tilde{q}_t(e)}$\;
Set $R_t(M(\tau_k))=2t^{\alpha}\mathbb{I}\{U_t=1\}\boldsymbol{\theta}_{M(\tau_k)}^\top\tilde{w}_t$ $+R_{t-1}(M(\tau_k))$,$\forall k\in[|\mathcal{M}|]$\;
Set $R_t(M(l))=2t^{\alpha}\mathbb{I}\{U_t=1\}\boldsymbol{\theta}_{M(l)}^\top\tilde{w}_t$ $+R_{t-1}(M(l))$,$\forall k\in[|\mathcal{M}_{KL}|]$\;
}
\KwOut{$\hat{\Delta}_{M,n}^{(i,j)}=\frac{1}{n}(R_n(M(\tau_i))-R_n(M(\tau_j))$, for $i,j\in[|\mathcal{M}|]$ and $i\neq j$; 
$\hat{\Delta}_{\mu,n}^{(i,j)}=\frac{1}{n}(R_n(M(i))-R_n(M(j))$, for $i,j\in[|\mathcal{M}_{KL}|]$ and $i\neq j$;}
\end{algorithm} 

The KL projection in \texttt{MixCombKL} ensures that $mq_{t-1}\in \mathrm{Co}(\boldsymbol{\theta}),$ and there exists $\psi$, a distribution over $\mathcal{M}$, such that $mq_{t-1}=\sum_{M}\psi(M)\boldsymbol{\theta}_M.$ This guarantees that the system of linear equations in the decomposition step is consistent. The algorithm of \cite{Sherali_1987} also demonstrates that the decomposition step can be efficiently implemented.

Second, to ensure the balance between exploration for inference and regret minimization, we set a random variable $U_t$ independently of previous history and the reward distribution, with probability $\mathbb{P}(U_t=0)=1-\frac{1}{2t^{\alpha}}$ and $\mathbb{P}(U_t=1)=\frac{1}{2t^{\alpha}}$. The pre-set parameter $\alpha\in[0,\frac{1}{2}]$ quantifies the decay of exploration. 
We select a distribution: for all $M\in\mathcal{M}$,
$$\pi_t(M)=p_{t-1}(M)\mathbb{I}\{U_t=0\}+|\mathcal{M}|^{-1}\mathbb{I}\{U_t=1\},$$
and sample a super arm according to \(p_{t-1}\). When choosing $\alpha=0$, the algorithm has equal probability at each time of choosing a uniform distribution or the KL-divergence-based distribution, and as $\alpha$ grows the algorithm tends to focus more on regret.
Thus, we can observe the full combinatorial feedback for the chosen super arm. Third, we update the empirical estimates of the reward parameters based on the observed feedback and compute an offline approximation oracle 
$$q_t=\mathbb{I}_{\{U_t=1\}}q_{t-1}+\mathbb{I}_{\{U_t=0\}}
\arg\min_{p\in\mathcal{Q}}\sum_{e\in\alpha}p(e)\log\frac{p(e)}{\tilde{q}_t(e)}$$ in order to find our pre-determined distribution $q_t$.
We also compute $R_t$ for future estimation of arms. Finally, we repeat the above steps for \(n\) rounds, dynamically refining the estimates and adjusting the distribution \(p_t\) to approach the Pareto-optimal trade-off between regret and estimation accuracy, and the estimation of basic actions and super arms is presented at the end.

Our algorithm departs from \cite{NIPS2015_0ce2ffd2} in that it improves inference accuracy by using a mixture of distributions as the super-arm sampling distribution, whereas prior work selects super arms solely based on distributions derived from KL-divergence properties.

Now, the following theorem provides an upper bound for the estimation error of super-arm gaps of \texttt{MixCombKL} through constructing a series of martingales.

\begin{theorem}
\label{Thm:delta1}
With probability at least $1-\delta$, it holds
$$|\hat{\Delta}_{M,n}^{(i,j)}-\Delta_M^{(i,j)}|\leq\frac{6}{\lambda_{\min}}\sqrt{\frac{m^{3}d}{n^{1-\alpha}}}\ln(\frac{2d}{\delta})$$
where $\Delta_M^{(i,j)}=\mathbb{E}[\hat{\Delta}_{M,n}^{(i,j)}]$ for any $i,j\in[|\mathcal{M}|]$ and $i\neq j$.
\end{theorem}
Note that, when taking $\delta=\frac{1}{n^2}$, we have 
$\max_{i<j\leq|\mathcal{M}|} \mathcal{E}(n,\hat{\Delta}_{M,n}^{(i,j)})=\tilde{\mathcal{O}}(\sqrt{n^{\alpha-1}})$. 
Similarly, we can derive an upper bound for the inference of $\Delta_{\mu}^{(i,j)}$.
\begin{corollary}
\label{Cor:1}
With probability at least $1-\delta$, it holds
$$|\hat{\Delta}_{\mu,n}^{(i,j)}-\Delta_{\mu}^{(i,j)}|\leq\frac{6m}{\lambda_{\min}}\sqrt{\frac{d}{n^{1-\alpha}}}\ln(\frac{2d}{\delta}),$$
where $\Delta_{\mu}^{(i,j)} = \mathbb{E}[\hat{\Delta}_{\mu,n}^{(i,j)}]$ for any $i,j\in[|\mathcal{M}_{KL}|]$ and $i\neq j$. 
\end{corollary}
Similarly, we have
$\max_{i<j\leq|\mathcal{M}_{KL}|} \mathcal{E}(n,\hat{\Delta}_{\mu,n}^{(i,j)})=\tilde{\mathcal{O}}(\sqrt{n^{\alpha-1}})$ by taking $\delta=\frac{1}{n^2}$.

We now turn to establishing the regret upper bound. Our algorithm introduces a sequence of independent Bernoulli random variables \(U_t\), with \(\mathbb{P}(U_t = 1) = \frac{1}{2 t^\alpha}\). When \(U_t = 1\), the algorithm performs uniform exploration over all super arms, ensuring equal sampling. This uniform exploration step contributes an additional \(\mathcal{O}(n^{1-\alpha})\) term to the regret bound originally derived in \cite{NIPS2015_0ce2ffd2}.

For the case when \(U_t = 0\), we establish a lemma of “triangle inequality” type, analogous to Lemma 4 in \cite{NIPS2015_0ce2ffd2}, which relates the KL divergence to the difference in expected reward between the optimal super arm and the arm selected by \texttt{MixCombKL}. This lemma allows us to effectively bound the regret and leads to the final regret guarantee for \texttt{MixCombKL}.

\begin{theorem}
\label{Thm:R1}
The regret satisfies
$$\begin{aligned}\mathcal{R}_{\nu}(n,\pi)\leq2\sqrt{m^{3}n\left(d+\frac{m^{1/2}}{\lambda_{\min}}\right)\log\rho_{\min}^{-1}}\\+\frac{mn^{1-\alpha}}{2(1-\alpha)}+\frac{m^{5/2}}{\lambda_{\min}}\log\rho_{\min}^{-1}.\end{aligned}$$   
\end{theorem}
For most classes of $\mathcal{M}$, we can observe that $\rho _{\mathrm{min}}^{- 1}= \mathcal{O}(\text{poly}(d))$ and $m(d\lambda_{\min})^{-1}=\mathcal{O}(1)$ \citep{Csiszar2004}. In these cases, we can derive that our \texttt{MixCombKL} algorithm has a regret bound of $\mathcal{O}(\sqrt{m^3dn\log(d/m)}+mn^{1-\alpha})$. This result indicates the effectiveness of our algorithm in tackling the intrinsic combinatorial nature of the full-bandit problem.
\subsection{Semi-Bandit Feedback}

For the semi-bandit setting, we relax the constraint on \( M \) to \(|M|\leq m\). The \texttt{MixCombUCB} algorithm (Algorithm \ref{alg:mixcombucb}) starts with the initialization procedure \texttt{InitUCB} (Algorithm \ref{alg:initucb}), which returns two variables. The first is a weight vector \( \hat{w} \in [0,1]^d \), where \( \hat{w}(e) \) is a single observation from the \( e \)-th marginal distribution \(\nu\). The second variable is the initialization step count plus one. \texttt{InitUCB} also outputs a set $E$ of observed basic actions and $m_0$ pairs of corresponding super arms for every $e\in E$ as $(e,M_e)$. 

To compute \( \hat{w} \), \texttt{InitUCB} repeatedly calls the optimization oracle \( M(t) = \arg\max_{M \in \mathcal{M}} \sum_{e \in M} u(e) \) on an auxiliary binary weight vector \( u \in \{0,1\}^d \) whose entries are initially set to all ones. When an item \( e \) is observed, \( \hat{w}(e) \) is set to the observed reward from arm $e$, while the respective \( u(e) \) is set to zero, and a corresponding super arm \( M_e \)  is recorded for future observations. The procedure terminates once all entries of \( u \) are zero.
Since at least one coordinate of \( u \) switches from one to zero in each iteration, \texttt{InitUCB} terminates in at most \( d \) steps. The number of initialization steps, denoted by \( m_0 \), corresponds to the number of basic arms included in \(\mathcal{M}\), which we denote by \( \mathcal{M}_{UCB} \)---the \emph{estimable basic actions} under the \emph{semi-bandit feedback} regime. 

At each time step \( t \), \texttt{MixCombUCB} proceeds in four phases. First, it computes the UCB for the expected weight of each item \( e \): $U_t(e) = \hat{w}_{T_{t-1}(e)}(e) + c_{t-1, T_{t-1}(e)},$
where \( \hat{w}_s(e) \) is the average of \( s \) observed weights for item \( e \), \( T_t(e) \) is the number of times item \( e \) has been observed up to time \( t \), and $c_{t,s} = \sqrt{\frac{2 \log t}{s}}$
is the confidence radius such that \( \mu(e) \in [\hat{w}_s(e) - c_{t,s}, \hat{w}_s(e) + c_{t,s}] \) holds with high probability.

\begin{algorithm}[htb]
\caption{\texttt{InitUCB}: \texttt{MixCombUCB} Initialization}
\label{alg:initucb}
\KwIn{$\mathcal{M}$;}
$\hat{\omega}(e)=0,u(e)=1,M_e=\emptyset,\;\forall e\in\mathcal{A};E=\emptyset,t=1$\;
\While{$\exists e\in \mathcal{A}:u(e)=1$}{
$M(t)=\arg\max_{M\in\mathcal{M}} \sum_{e \in M} u(e)$\;
Observe $\{(e,w_{t}(e)):e \in M(t)\}$ where $w_{t}\sim \nu$\;
\ForAll {$e\in M(t)$}{
$\hat{\omega}(e)=\omega_t(e),u(e)=0,E=E\cup\{e\}$\;
$M_e=M(t),t=t+1$\;
}
}
\KwOut{$E,\hat{\omega},t,\forall e\in E,(e,M_e)$;}
\end{algorithm}

Second, \texttt{MixCombUCB} queries the optimization oracle to solve the combinatorial maximization problem over the UCB super arm $\tilde{M}(t) = \arg\max_{M \in \mathcal{M}} f(M, U_t)$.

Third, we set $\alpha_t=\frac{1}{m_0t^{\alpha}}$, then construct a probability distribution over super arms: for all $M \in \mathcal{M},$
$$\pi_t(M) = (1 - m_0 \alpha_t) \mathbb{I}\{ M = \tilde{M}(t) \} + \sum_{e=1}^{m_0} \alpha_t \mathbb{I}\{ M = M_e \}.$$

Then the algorithm selects \( M_t \) according to $\pi_t(M)$. Unlike the fully combinatorial case, the range of $\alpha$ is selected with respect to different classes of available action spaces $\mathcal{M}$, and we explain the $\alpha$ selection range based on different basic action gap properties below.

For each suboptimal basic action $e \in \tilde{\mathcal{A}}$, where $\tilde{\mathcal{A}} = \mathcal{A} \setminus \mathcal{M}^*$ denotes the set of items not included in any optimal super arm, we define the \textit{minimum gap} $\Delta_{e,\min}$:  
\begin{equation}
\label{Eq:deltagapdif}
\Delta_{e,\min} = f(M^*,\boldsymbol{\mu}) - \max_{M \in \mathcal{M}: e \in M,\ \Delta_{M} > 0} f(M,\boldsymbol{\mu}) .
\end{equation}  
We say the \textit{large-gap property} holds when $\Delta_{e,\min} = \Theta(1)$ for all $e \in \mathcal{A}$. This means that no suboptimal action is ``nearly optimal'': the performance gap between any suboptimal action and the optimal choice is bounded away from zero by a constant, independent of the time horizon $n$ or the problem size. In our algorithm, $\alpha$ can be chosen between $[0,1]$ when the large-gap property holds, and  $\alpha\in[0,\tfrac{1}{2}]$ when the large-gap property does not hold. 

Fourth, \texttt{MixCombUCB} observes the weights of all selected items and then updates the estimates \( \hat{w}(e) \) and arms estimator $R_t$ accordingly. The estimation of basic actions and super arms is finally given after running for $n$ time steps. The pseudocode for \texttt{MixCombUCB} is presented in Algorithm \ref{alg:mixcombucb}.

Our algorithmic framework differs from \cite{pmlr-v38-kveton15}, where the UCB-optimal super arm is always selected. By introducing randomness in the selection process, our approach achieves a better balance between minimizing inference error and regret. We begin by presenting a theorem on estimation quality, derived through a martingale construction.

\begin{algorithm}[htb]
\caption{Mixture-Based Combinatorial UCB Algorithm (\texttt{MixCombUCB})}  \label{alg:mixcombucb}  
\KwIn{$\alpha,\mathcal{M},n$;}
$\{E_1,\hat{\omega}_1,m_0,(1,M_1),...,(m_0,M_{m_0})\}$$=$$\texttt{InitUCB}(\mathcal{M})$,
$T_{m_0-1}(e)=1,R_{m_0-1}(e)=0, \forall e\in\mathcal{A}$\;
\For{$t=m_0,...,n$}{
Set $\alpha_{t}=\frac{1}{m_0t^{\alpha}}$\;
$U_{t}(e)=\hat{w}_{T_{t-1}(e)}(e)+c_{t-1,T_{t-1}(e)},\forall e\in \mathcal{A}$\;
Solve optimization problem $\tilde{M}(t)=\arg\max_{M\in\mathcal{M}}f(  M,U_{t})$\;
Set $\pi_t(M)=(1-m_0\alpha_t)\mathbb{I}\{M=\tilde{M}(t)\}+\sum_{e=1}^{m_0}\alpha_t\mathbb{I}\{M=M_e\},\forall M\in\mathcal{M}$\;
Select $M(t)$ according to $\pi_t$, Observe $\{(e,w_{t}(e)):e \in M(t)\}$\;
$R_{t}(e)=R_{t-1}(e)+w_{t}(e)\frac{\mathbb{I}\{e \in M(t)\}}{\mathbb{P}(e \in M(t))},\forall e\in\mathcal{A}$\;
Set $T_{t}(e)= T_{t-1}(e),\forall e\in \mathcal{A},T_{t}(e)= T_{t-1}(e)+1,\forall e \in M(t)$\;
$\hat{w}_{T_{t}(e)}(e)=\frac{T_{t-1}(e)\hat{w}_{T_{t-1}(e)}(e)+w_{t}(e)}{T_{t}(e)},\forall e\in M(t)$\;
}
\KwOut{$\hat{\Delta}_{M,n}^{(i,j)}=\frac{1}{n}(\sum_{e \in M(\tau_i)}R_t(e)-\sum_{e \in M(\tau_j)}R_t(e))$
for $i,j\in[|\mathcal{M}|]$ and $i\neq j$; $\hat{\Delta}_{\mu,n}^{(i,j)}=\frac{1}{n}(\sum_{e \in M(i)}R_t(e)-\sum_{e \in M(j)}R_t(e))$ for $i,j\in[|\mathcal{M}_{UCB}|]$ and $i\neq j$;}
\end{algorithm}

\begin{theorem}
\label{Thm:delta2}
With probability at least $1-\delta$, it holds
$$|\hat{\Delta}_{M,n}^{(i,j)}-\Delta_M^{(i,j)}|\leq6md\sqrt{n^{\alpha-1}}\ln(\frac{2d}{\delta}),$$
where $\mathbb{E}[\hat{\Delta}_{M,n}^{(i,j)}] = \Delta_M^{(i,j)}$ for any $i,j\in[|\mathcal{M}|]$ and $i\neq j$.
\end{theorem}
We have 
$\max_{i<j\leq|\mathcal{M}|} \mathcal{E}(n,\hat{\Delta}_{M,n}^{(i,j)})=\tilde{\mathcal{O}}(\sqrt{n^{\alpha-1}})$ after taking $\delta=\frac{1}{n^2}$.   
The corresponding error bound on the inference of $\Delta_{\mu}^{(i,j)}$ is characterized in the following Corollary.
\begin{corollary}
\label{Cor:2}
With probability at least $1-\delta$, it holds
$$|\hat{\Delta}_{\mu,n}^{(i,j)}-\Delta_{\mu}^{(i,j)}|\leq6d\sqrt{n^{\alpha-1}}\ln(\frac{2d}{\delta}),$$
where $\mathbb{E}[\hat{\Delta}_{\mu,n}^{(i,j)}] = \Delta_{\mu}^{(i,j)}$ , $\forall i,j\in[|\mathcal{M}_{UCB}|]$ and $i\neq j$.       
\end{corollary}
It implies that when taking $\delta=\frac{1}{n^2}$, we have
$\max_{i<j\leq|\mathcal{M}_{UCB}|}\mathcal{E}(t,\hat{\Delta}_{\mu,n}^{(i,j)})=\tilde{\mathcal{O}}(\sqrt{n^{\alpha-1}})$.

Now we proceed with the regret analysis. Given the large-gap property result, we can derive a sharper regret bound, since a non-vanishing gap allows the algorithm to distinguish suboptimal actions more quickly. We formalize this in the following theorem:
\begin{theorem}
\label{Thm:R2}
Given the knowledge of gaps between basic actions, the regret of our algorithm is bounded by 
\begin{equation*}
\mathcal{R}_{\nu}(n,\pi)\leq\sum_{e\in\tilde{\mathcal{A}}}\frac{712m}{\Delta_{e,\min}}\log n+5md+\frac{mn^{1-\alpha}}{2(1-\alpha)},
\end{equation*}
where $\Delta_{e,\min}$ is the minimum gap of suboptimal super arms that contain item $e$,  which is defined in \eqref{Eq:deltagapdif}. 
\end{theorem}

It shows that if for all $e\in\mathcal{A}$, we have $\Delta_{e,\min}=\Theta(1)$, then the regret of \texttt{MixCombUCB} is bounded by $\mathcal{O}(md\log n+mn^{1-\alpha})$. There's also a gap-free bound:
\begin{proposition}
\label{Thm:R3}
For all classes of $\mathcal{M}$, the regret of our algorithm is bounded as:
\begin{equation*}
\mathcal{R}_{\nu}(n,\pi)\leq54\sqrt{mdn\log n}+5md+\frac{mn^{1-\alpha}}{2(1-\alpha)}.    
\end{equation*}
\end{proposition}
This result demostrates that even without the large-gap property, an $\tilde{\mathcal{O}}(\sqrt{n}+n^{1-\alpha})$ regret is still achievable, and it guarantees the exploration of basic action distributions even without assumptions on arms behaviors.

\section{Pareto Optimality Conditions}
\label{sec:Paretocondition}
We now present the conditions for Pareto optimality and verify that our algorithm achieves it. 
Our results in this section do not differentiate full- and semi-bandit settings, as our proofs are agnostic to the feedback model. 
We present the sufficient and necessary conditions for Pareto optimality for any CMAB algorithm. We first establish the equivalent condition of Pareto optimality based on super arm $M$. This setting is comparable to the classic $K$-arm MAB problem \citep{doi:10.1287/mnsc.2023.00492}; here each super arm $M$ has reward $\sum_{e\in M}\boldsymbol{\mu}(e)$.

\begin{theorem}
\label{Thm:ParetoM}
The necessary and sufficient condition for an admissible pair $(\pi,\hat{\Delta}_M)$ to be Pareto optimal is
\begin{equation*}
\max_{\nu\in\mathcal{V}_0}\left[(\max_{i<j\leq |\mathcal{M}|}\mathcal{E}(n,\hat{\Delta}_{M,n}^{(i,j)}))\sqrt{\mathcal{R}_{\nu}(n,\pi)}\right]=\widetilde{\mathcal{O}}(1)    
\end{equation*}
\end{theorem}

We also extend the trade-off to inference for base arms, where we consider the estimable basic actions set $\mathcal{A}_{ad}$ as aforementioned. Note that in full-bandit setting, $\mathcal{A}_{ad}= \mathcal{M}_{KL}$ and in semi-bandit setting, $\mathcal{A}_{ad}= \mathcal{M}_{UCB}$. Formally, the correponding sufficient and necessary condition can be characterized as follows.
\begin{theorem}
\label{Thm:Paretomu}
The necessary and sufficient condition for an admissible pair $(\pi,\hat{\Delta}_\mu)$ to be Pareto optimal is
\begin{equation*}
\max_{\nu} \left( \max_{i < j \leq |\mathcal{A}_{\mathrm{ad}}|} \mathcal{E}(n, \widehat{\Delta}_\mu^{(i,j)}) \right) \sqrt{\mathcal{R}_\nu(n, \pi)} = \widetilde{\mathcal{O}}(1).    
\end{equation*}
\end{theorem}

We provide the proof sketch for our Theorem~\ref{Thm:Paretomu}. We start with the simple case of $|\mathcal{A}_{ad}|=2$ by proving the information-theoretic bound (in Lemma \ref{Thm:Paretoinf}), then extend to general case (in Lemma \ref{lemma: general}).  
\begin{lemma}
\label{Thm:Paretoinf}
When $|\mathcal{A}_{ad}|=2$, the lower bound for all admissible pair $(\pi,\hat{\Delta}_{\mu})$ has
\begin{equation}
\inf_{(\pi,\widehat\Delta_{\mu})}\max_{\nu\in\mathcal{V}_0}\left[\mathcal{E}(n,\widehat\Delta_{\mu})\sqrt{\mathcal{R}_{\nu}(n,\pi)}\right]=\Omega(1),
\end{equation}
for any selected algorithm.
\end{lemma}
Above lemma demostrates that $\mathcal{E}(n,\widehat\Delta_{\mu})\sqrt{\mathcal{R}_{\nu}(n,\pi)}$ will not perform better than a constant order in the worse case, no matter how the solution is chosen. Particularily, we are interested in how the upper bound of $\mathcal{E}(n,\widehat\Delta_{\mu})\sqrt{\mathcal{R}_{\nu}(n,\pi)}$ influence our Pareto optimality condition. The following lemma shows that the sufficient condition for Pareto optimality holds when $\mathcal{E}(n,\widehat\Delta_{\mu})\sqrt{\mathcal{R}_{\nu}(n,\pi)}$ achieves a constant order. 
\begin{lemma}
\label{Thm:Paretosup}
When $|\mathcal{A}_{ad}|=2$, the sufficient condition for an admissible pair $(\pi,\hat{\Delta}_{\mu})$ to be Pareto optimal is
\begin{equation*}
\max_{\nu\in\mathcal{V}_0}\left[\mathcal{E}(n,\widehat\Delta_{\mu})\sqrt{\mathcal{R}_{\nu}(n,\pi)}\right]=\widetilde{\mathcal{O}}(1).   
\end{equation*}
\end{lemma}

\begin{table*}[htb!]
\centering
\caption{Pareto Frontier restrictions for fully and semi-bandit feedback models. 
We consider the estimation for both super arms and basic actions and we set 
$\lambda_{\min}$ for most instances of $\mathcal{M}$ so that $\lambda_{\min}^{-1} = \mathcal{O}(d m^{-1})$ \citep{CESABIANCHI20121404}. 
The parameter $\alpha$ is chosen within the admissible range for each algorithm.}
\label{tab:ptf_comp}
\begin{tabular}{|l|cc|cc|}
\hline
Feedback model & \multicolumn{2}{c|}{Full-Bandit} & \multicolumn{2}{c|}{Semi-Bandit} \\ \hline
Admissible pair & \multicolumn{1}{c|}{$(\pi,\hat\Delta_{M,n})$} & $(\pi,\hat\Delta_{\mu,n})$ 
                & \multicolumn{1}{c|}{$(\pi,\hat\Delta_{M,n})$} & $(\pi,\hat\Delta_{\mu,n})$ \\ \hline
Inference error $e(n,\hat{\Delta})$ & \multicolumn{1}{c|}{$\tilde{\mathcal{O}}(\sqrt{md^3n^{\alpha-1}})$} 
                                    & $\tilde{\mathcal{O}}(\sqrt{d^3n^{\alpha-1}})$ 
                                    & \multicolumn{1}{c|}{$\tilde{\mathcal{O}}(md\sqrt{n^{\alpha-1}})$} 
                                    & $\tilde{\mathcal{O}}(d\sqrt{n^{\alpha-1}})$ \\ \hline
Regret bound $\mathcal{R}_{\nu}(\pi,n)$ & \multicolumn{1}{c|}{$\tilde{\mathcal{O}}(mn^{1-\alpha})$} 
                                        & $\tilde{\mathcal{O}}(mn^{1-\alpha})$ 
                                        & \multicolumn{1}{c|}{$\tilde{\mathcal{O}}(mn^{1-\alpha})$} 
                                        & $\tilde{\mathcal{O}}(mn^{1-\alpha})$ \\ \hline
Pareto Frontier $S_{PF}$ & \multicolumn{1}{c|}{$\tilde{\mathcal{O}}(m\sqrt{d^3})$} 
                         & $\tilde{\mathcal{O}}(\sqrt{md^3})$ 
                         & \multicolumn{1}{c|}{$\tilde{\mathcal{O}}(d\sqrt{m^3})$} 
                         & $\tilde{\mathcal{O}}(d\sqrt{m})$ \\ \hline
\end{tabular}
\end{table*}
Now we extend our results to the general setting. We have 
$\max_{i < j \leq |\mathcal{A}_{\mathrm{ad}}|} \mathcal{E}(n, \Delta_\mu^{(i,j)}) = \tilde{\mathcal{O}}\left( \sqrt{n^{\alpha - 1}} \right)$ from ~\ref{Cor:1} and~\ref{Cor:2}. 
Combining this with Theorem ~\ref{Thm:R1} and~\ref{Thm:R3}, it follows that for any combinatorial bandit instance $\nu$,
$\left( \max_{i < j \leq |\mathcal{A}_{\mathrm{ad}}|} \mathcal{E}(n, \widehat{\Delta}_\mu^{(i,j)}) \right) \cdot \sqrt{\mathcal{R}_\nu(n, \pi)} = \widetilde{\mathcal{O}}(1).$ 
This observation allows us to generalize Lemma~\ref{Thm:Paretosup}, yielding a \emph{sufficient condition} for the general case:
$\max_{\nu} \left( \max_{i < j \leq |\mathcal{A}_{\mathrm{ad}}|} \mathcal{E}(n, \widehat{\Delta}_\mu^{(i,j)}) \right) \sqrt{\mathcal{R}_\nu(n, \pi)} = \widetilde{\mathcal{O}}(1)$. Moreover, combining the sufficient condition with the definition of Pareto optimality, we can proceed to the following lemma that demostrates the \emph{necessary condition} for Pareto optimality by contradiction:
\begin{lemma}\label{lemma: general}
The necessary condition for an admissible pair $(\pi,\hat{\Delta}_\mu)$ to be Pareto optimal is equivalent to
\begin{equation}
\max_{\nu} \left( \max_{i < j \leq |\mathcal{A}_{\mathrm{ad}}|} \mathcal{E}(n, \widehat{\Delta}_\mu^{(i,j)}) \right) \sqrt{\mathcal{R}_\nu(n, \pi)} = \widetilde{\mathcal{O}}(1).    
\end{equation}
\end{lemma}

Having Lemma \ref{Thm:Paretosup} and \ref{lemma: general}, the proof for Theorem \ref{Thm:Paretomu} follows directly. 
Finally, given Theorem~\ref{Thm:Paretomu} and \ref{Thm:Paretomu}, we are now able to state that our algorithms \texttt{\texttt{MixCombKL}} and \texttt{\texttt{MixCombUCB}} are Pareto optimal:

\begin{theorem}
\texttt{\texttt{MixCombKL}} and \texttt{\texttt{MixCombUCB}} are Pareto optimal when choosing $\alpha\in[0,\frac{1}{2}]$ in full- and semi-bandit settings. Specifically, \texttt{\texttt{MixCombUCB}} is Pareto optimal for all $\alpha\in[0,1]$ if the large gap property $\Delta_{e,\min}=\Theta(1),\forall e\in\mathcal{A}$ holds.
\end{theorem}

It is clear that the decision-maker has no incentive to select an $\alpha$ exceeding the specified upper bound, as doing so increases estimation error while offering negligible statistical benefit in terms of the regret upper bound. Thus, we successfully prove that our algorithms are Pareto optimal. Moreover, while the proofs of Pareto optimality in the full- and semi-bandit settings share a high-level structure, the differences in feedback lead to notable variations in parameter constraints. The reason is that the range of the exploration parameter $\alpha$ depends on the amount of information available from feedback. In the full-bandit case, only the total reward of super arms is observed. Over-exploring (choosing a large $\alpha$) can significantly increase estimation error relative to any gains in regret reduction, limiting $\alpha$ to the range $[0, \frac{1}{2}]$. In contrast, in the semi-bandit case, the reward of each basic action is observed directly, providing richer and more precise feedback. This allows $\alpha$ to safely increase up to $1$ when the large-gap property holds, enhancing regret performance, while still maintaining a conservative range of $\alpha \in [0, \frac{1}{2}]$ if the large-gap property does not hold.

\section{Pareto Frontiers of Different Feedback Structures}
\label{sec:pf_comp}
We consider the Pareto frontiers achieved by \texttt{MixCombKL} and \texttt{MixCombUCB} in the two bandit feedback settings. From Theorems \ref{Thm:delta1}, \ref{Thm:R1}, \ref{Thm:delta2}, and \ref{Thm:R2}, for any admissible pair $(\pi,\hat{\Delta}_M),(\pi,\hat{\Delta}_\mu)$ and reward distribution $\nu \in \mathcal{V}_0$, \texttt{MixCombKL} and \texttt{MixCombUCB} can achieve Pareto optimality. We can generally write their respective Pareto frontier $\mathcal{P}_f(\mathrm{Alg},\pi,\hat{\Delta})$ with respective algorithm and admissible pair $(\pi,\hat{\Delta}_{n})$ in the following form:  
\[
\left\{ (\pi,\hat{\Delta}_{n}) \;\middle|\; \left(\max_{i<j} e_{\nu}(n, \hat{\Delta}_{n}^{(i,j)})\right) \sqrt{\mathcal{R}_{\nu}(n,\pi)} = S_{PF}\right\},
\]  
where $S_{PF}$ is the rate of the Pareto frontier for each algorithm and admissible pair. 

In Table~\ref{tab:ptf_comp}, we compare the rate differences of Pareto frontiers of our algorithms under feedback models with different information structures.
We consider the Pareto frontiers for both super arms and base arms. For fair comparison, we set  $\lambda_{\min}^{-1} = \mathcal{O}(d m^{-1})$ \citep{CESABIANCHI20121404}. 
These results indicate that the Pareto frontier of the pair $(\pi,\hat{\Delta}_{M})$ achieved by \texttt{MixCombUCB} is $\tilde{\mathcal{O}}(\sqrt{d/m})$ tighter than that of \texttt{MixCombKL}.  

\begin{figure}[htbp]
    \centering
    \includegraphics[width=0.48\textwidth]{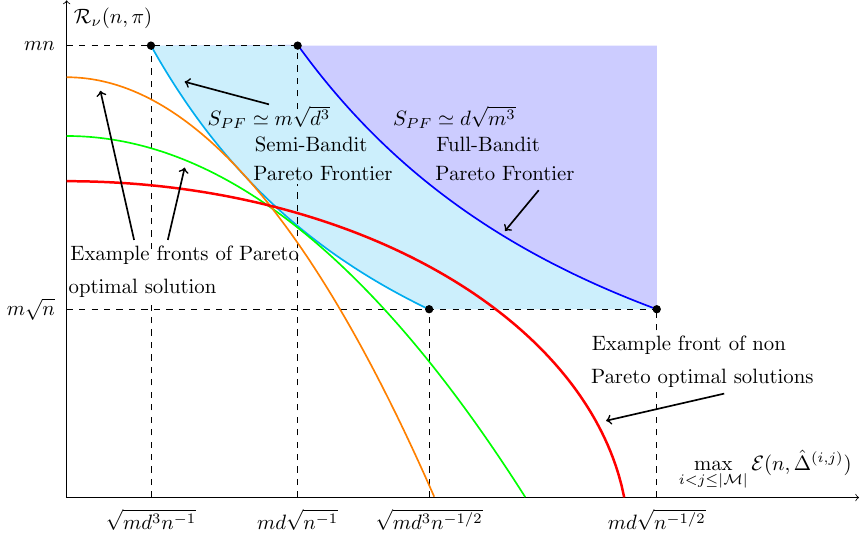}
    \caption{Summary of $\hat{\Delta}_{M,n}$ Pareto Frontier Results with Feedback Difference.}
    \label{fig:pareto_frontier}
\end{figure}

Figure \ref{fig:pareto_frontier} shows the Pareto frontier results under different feedback models.  
This difference stems from the richness of the bandit feedback. In the semi-bandit case, access to individual arm rewards provides more precise information, yielding a more favorable trade-off between regret and estimation error. In contrast, in the full-bandit case, limited feedback drives a worse Pareto frontier and constrains achievable trade-offs. Thus, semi-bandit algorithms attain more accurate estimations due to improved access to arm information. Yet, the regret in both cases is largely driven by random exploration under appropriate $\alpha$, and in expectation this step occurs $\mathcal{O}(n^{1-\alpha})$ times under both algorithms, overshadowing the regret from arm estimation and keeping both cases at the $\mathcal{O}(mn^{1-\alpha})$ level.

\section{Experiment}

We evaluate the empirical performance of \texttt{MixCombKL} and \texttt{MixCombUCB} under synthetic full-bandit and semi-bandit feedback settings with time horizon $n$. We use the following metrics to evaluate the algorithms: 
\begin{itemize}
    \item $R(n)=\sum_{t=1}^{n}[\mathbb{E}[f(M^{*},w_{t})]-\mathbb{E}^{\pi}\left[f(M^{\pi}(t),w_{t})\right]]$
    \item $MSE_\mu=\frac{2}{d(d-1)}\sum_{1\leq i<j\leq d}(\hat{\Delta}_{\mu,n}^{(i,j)}-\Delta_{\mu}^{(i,j)})^2$
    \item $MSE_M=\frac{2}{|\mathcal{M}|^2-|\mathcal{M}|}\sum_{1\leq i<j\leq |\mathcal{M}|}(\hat{\Delta}_{M,n}^{(i,j)}-\Delta_{M}^{(i,j)})^2$
\end{itemize}

Here $R(n),MSE_\mu$ and $MSE_{M}$ represent the cumulative regret, mean square error of base arms and super arms. For each trial, we draw the expected reward $\mu(e)\sim\mathcal{U}(0.1,0.9)$ for all $e\in[d]$ independently and allow any non-empty super-arm size at $m$. We repeat over 20 independent sampling trials for regret and gap estimation error analysis. 

\paragraph{Result of MixCombKL Experiment}
Our method applies forced sampling with decay $\alpha_{t}=\frac{1}{2t^{\alpha}}$, $\alpha\in\{0,0.25,0.5,1\}$. As the KL-projection based algorithm is harder to converge, we set $d=8,m=3$ and have in total $\binom{8}{3}=56$ super arms. We run experiments for $n=5000$ steps.

\begin{figure}[htbp]
    \centering
    \includegraphics[width=0.48\textwidth]{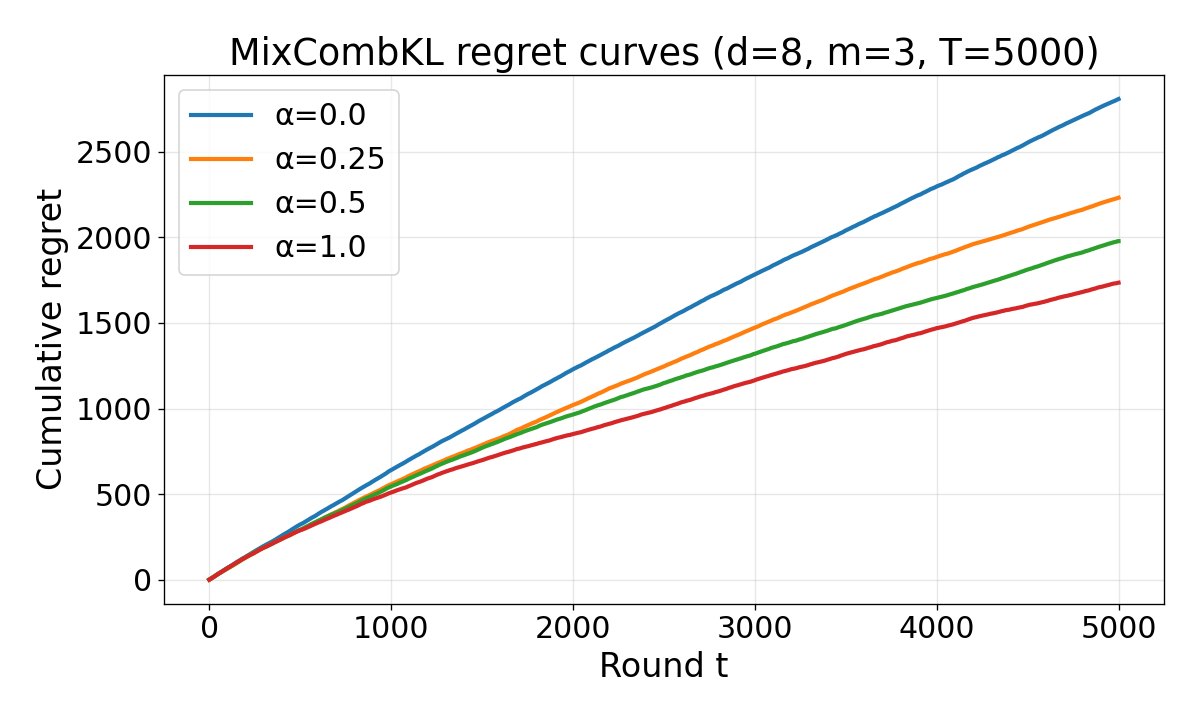}
    \includegraphics[width=0.48\textwidth]{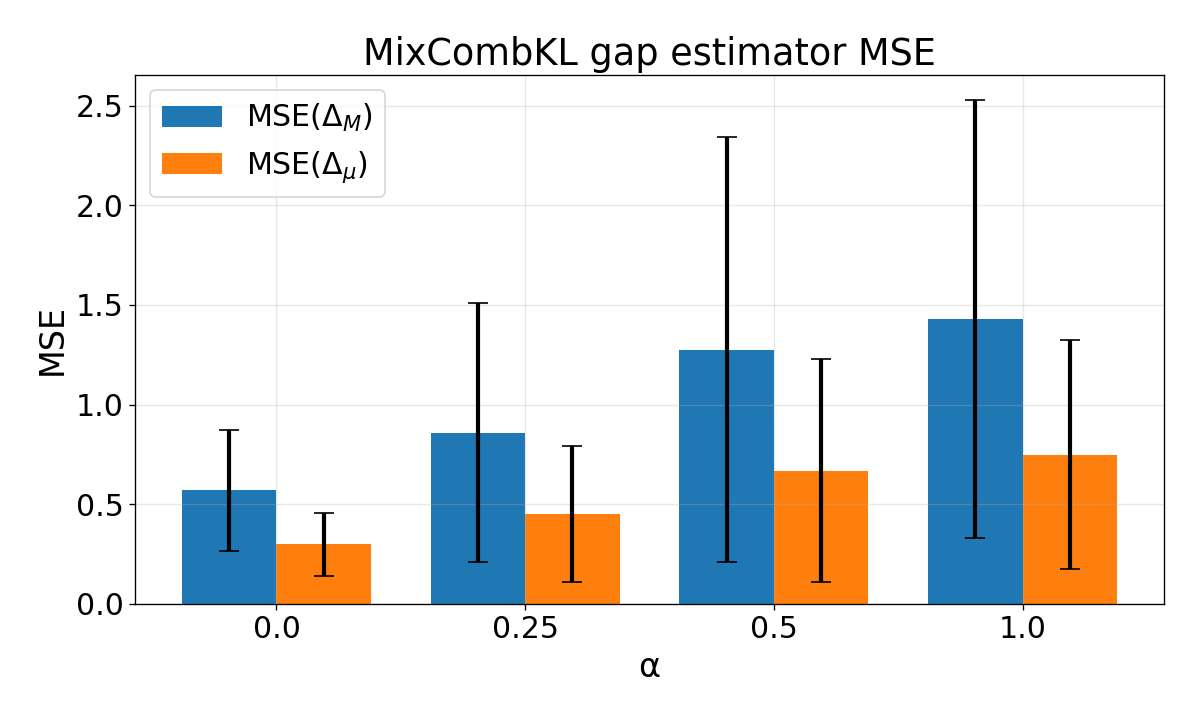}
    \caption{The regret and MSE plot of \texttt{MixCombKL}.}
    \label{fig:kl}
\end{figure} 

\paragraph{Result of MixCombUCB Experiment}
Similarily, our method applies forced sampling with decay $\alpha_{t}=\frac{1}{m_{0}t^{\alpha}}$, 
$\alpha\in\{0,0.25,0.5,1\}$. Here we set $d=9,m=4$ and we have in total $\binom{9}{4}=126$ super arms. We run experiments for $n=2000$ steps.

\begin{figure}[htbp]
    \centering
    \includegraphics[width=0.48\textwidth]{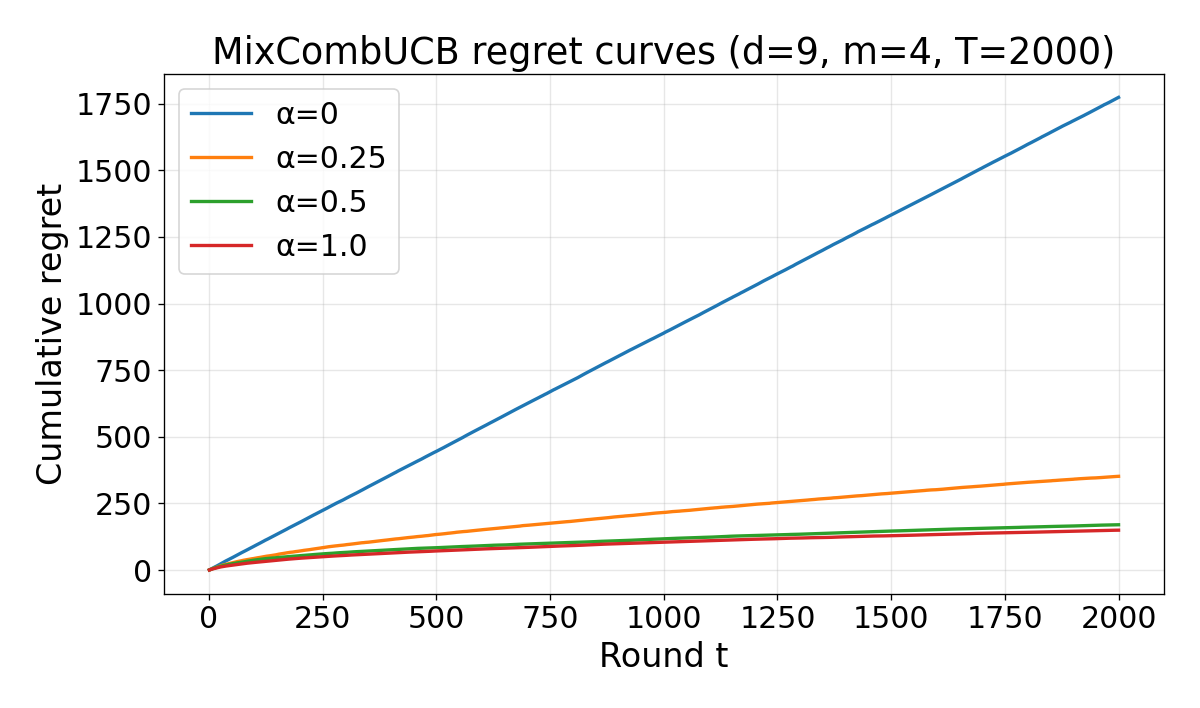}
    \includegraphics[width=0.48\textwidth]{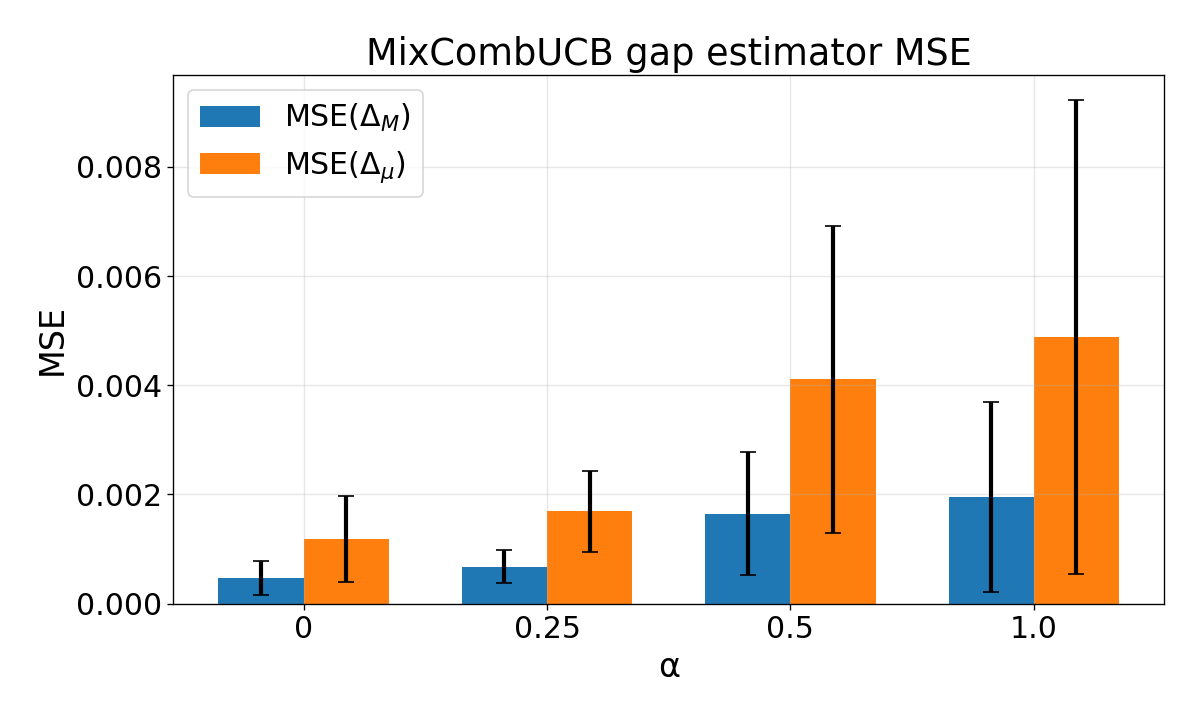}
    \caption{The regret and MSE plot of \texttt{MixCombUCB}.}
    \label{fig:ucb}
\end{figure}

\section{Conclusion}
In this paper, we investigate the concept of Pareto optimality for combinatorial bandits, framing it as the fundamental trade-off between minimizing cumulative regret and reducing average estimation errors of both base and super arm reward gaps. We characterize the sufficient and necessary conditions for achieving Pareto optimality in both full- and semi-bandit feedback settings. We propose two novel algorithms \texttt{MixCombKL} for the full-bandit setting and \texttt{MixCombUCB} for the semi-bandit case, and show that both algorithms are provably Pareto optimal. Looking ahead, future work could extend our Pareto-optimal framework to dynamic combinatorial settings and investigate Pareto optimality under alternative performance metrics, such as regret variants or average treatment effect (ATE), that better capture practical considerations in complex environments. Another valuable direction is to incorporate constraints (e.g., budgets or fairness) into the trade-off analysis, broadening the framework’s applicability to real-world decision systems.

\newpage

{
\bibliography{refs}

@article{Csiszar2004,
	title        = {Information Theory and Statistics: A Tutorial},
	author       = {Csiszar, Imre and Shields, Paul},
	year         = 2004,
	month        = {01},
	journal      = {Foundations and Trends in Communications and Information Theory},
	volume       = 1,
	doi          = {10.1561/0100000004}
}

@article{CESABIANCHI20121404,
	title        = {Combinatorial bandits},
	author       = {Nicolò Cesa-Bianchi and Gábor Lugosi},
	year         = 2012,
	journal      = {Journal of Computer and System Sciences},
	volume       = 78,
	number       = 5,
	pages        = {1404--1422},
	doi          = {https://doi.org/10.1016/j.jcss.2012.01.001},
	issn         = {0022-0000},
	url          = {https://www.sciencedirect.com/science/article/pii/S0022000012000219},
	note         = {JCSS Special Issue: Cloud Computing 2011}
}

@article{Sherali_1987,
	title        = {A constructive proof of the representation theorem for polyhedral sets based on fundamental definitions},
	author       = {Sherali, Hanif D.},
	year         = 1987,
	month        = feb,
	journal      = {American Journal of Mathematical and Management Sciences},
	volume       = 7,
	number       = {3–4},
	pages        = {253–270},
	doi          = {10.1080/01966324.1987.10737221},
	url          = {https://www.tandfonline.com/doi/abs/10.1080/01966324.1987.10737221}
}

@inproceedings{10.5555/3020751.3020795,
	title        = {Matroid bandits: fast combinatorial optimization with learning},
	author       = {Kveton, Branislav and Wen, Zheng and Ashkan, Azin and Eydgahi, Hoda and Eriksson, Brian},
	year         = 2014,
	booktitle    = {Proceedings of the Thirtieth Conference on Uncertainty in Artificial Intelligence},
	location     = {Quebec City, Quebec, Canada},
	publisher    = {AUAI Press},
	address      = {Arlington, Virginia, USA},
	series       = {UAI'14},
	pages        = {420–429},
	isbn         = 9780974903910,
	numpages     = 10
}

@inproceedings{NIPS2015_0ce2ffd2,
	title        = {Combinatorial Bandits Revisited},
	author       = {Combes, Richard and Talebi Mazraeh Shahi, Mohammad Sadegh and Proutiere, Alexandre and lelarge, marc},
	year         = 2015,
	booktitle    = {Advances in Neural Information Processing Systems},
	publisher    = {Curran Associates, Inc.},
	volume       = 28,
	url          = {https://proceedings.neurips.cc/paper_files/paper/2015/file/0ce2ffd21fc958d9ef0ee9ba5336e357-Paper.pdf},
	editor       = {C. Cortes and N. Lawrence and D. Lee and M. Sugiyama and R. Garnett}
}

@book{Boyd_Vandenberghe_2004,
	title        = {Convex Optimization},
	author       = {Boyd, Stephen and Vandenberghe, Lieven},
	year         = 2004,
	publisher    = {Cambridge University Press},
	place        = {Cambridge}
}

@article{doi:10.1287/mnsc.2023.00492,
	title        = {Multi-armed Bandit Experimental Design: Online Decision-Making and Adaptive Inference},
	author       = {Simchi-Levi, David and Wang, Chonghuan},
	year         = 2025,
	journal      = {Management Science},
	volume       = 71,
	number       = 6,
	pages        = {4828--4846},
	doi          = {10.1287/mnsc.2023.00492},
	url          = {https://doi.org/10.1287/mnsc.2023.00492},
	eprint       = {https://doi.org/10.1287/mnsc.2023.00492}
}

@inproceedings{pmlr-v38-kveton15,
	title        = {{Tight Regret Bounds for Stochastic Combinatorial Semi-Bandits}},
	author       = {Kveton, Branislav and Wen, Zheng and Ashkan, Azin and Szepesvari, Csaba},
	year         = 2015,
	month        = {09--12 May},
	booktitle    = {Proceedings of the Eighteenth International Conference on Artificial Intelligence and Statistics},
	publisher    = {PMLR},
	address      = {San Diego, California, USA},
	series       = {Proceedings of Machine Learning Research},
	volume       = 38,
	pages        = {535--543},
	url          = {https://proceedings.mlr.press/v38/kveton15.html},
	editor       = {Lebanon, Guy and Vishwanathan, S. V. N.}
}

@article{14982149-b382-3e1b-8d96-b63be1f53b29,
	title        = {Regret in Online Combinatorial Optimization},
	author       = {Jean-Yves Audibert and Sébastien Bubeck and Gábor Lugosi},
	year         = 2014,
	journal      = {Mathematics of Operations Research},
	publisher    = {INFORMS},
	volume       = 39,
	number       = 1,
	pages        = {31--45},
	issn         = {0364765X, 15265471},
	url          = {http://www.jstor.org/stable/24540886},
	urldate      = {2025-08-16}
}

@misc{zhong2023achievingparetofrontierregret,
	title        = {Achieving the Pareto Frontier of Regret Minimization and Best Arm Identification in Multi-Armed Bandits},
	author       = {Zixin Zhong and Wang Chi Cheung and Vincent Y. F. Tan},
	year         = 2023,
	url          = {https://arxiv.org/abs/2110.08627},
	eprint       = {2110.08627},
	archiveprefix = {arXiv},
	primaryclass = {cs.LG}
}

@inproceedings{pmlr-v65-chen17b,
	title        = {Towards Instance Optimal Bounds for Best Arm Identification},
	author       = {Chen, Lijie and Li, Jian and Qiao, Mingda},
	year         = 2017,
	month        = {07--10 Jul},
	booktitle    = {Proceedings of the 2017 Conference on Learning Theory},
	publisher    = {PMLR},
	series       = {Proceedings of Machine Learning Research},
	volume       = 65,
	pages        = {535--592},
	url          = {https://proceedings.mlr.press/v65/chen17b.html},
	editor       = {Kale, Satyen and Shamir, Ohad}
}

@article{10.5555/2946645.2946646,
	title        = {On the complexity of best-arm identification in multi-armed bandit models},
	author       = {Kaufmann, Emilie and Capp\'{e}, Olivier and Garivier, Aur\'{e}lien},
	year         = 2016,
	month        = jan,
	journal      = {J. Mach. Learn. Res.},
	publisher    = {JMLR.org},
	volume       = 17,
	number       = 1,
	pages        = {1–42},
	issn         = {1532-4435},
	issue_date   = {January 2016},
	numpages     = 42
}

@inproceedings{pmlr-v89-degenne19a,
	title        = {Bridging the gap between regret minimization and best arm identification, with application to A/B tests},
	author       = {Degenne, R\'emy and Nedelec, Thomas and Calauzenes, Clement and Perchet, Vianney},
	year         = 2019,
	month        = {16--18 Apr},
	booktitle    = {Proceedings of the Twenty-Second International Conference on Artificial Intelligence and Statistics},
	publisher    = {PMLR},
	series       = {Proceedings of Machine Learning Research},
	volume       = 89,
	pages        = {1988--1996},
	url          = {https://proceedings.mlr.press/v89/degenne19a.html},
	editor       = {Chaudhuri, Kamalika and Sugiyama, Masashi}
}

@inproceedings{audibert:hal-00654404,
	title        = {{Best Arm Identification in Multi-Armed Bandits}},
	author       = {Audibert, Jean-Yves and Bubeck, S{\'e}bastien},
	year         = 2010,
	month        = {June},
	booktitle    = {{COLT 2010 - Proceedings}},
	address      = {Haifa, Israel},
	url          = {https://enpc.hal.science/hal-00654404},
	hal_id       = {hal-00654404},
	hal_version  = {v1}
}

@inproceedings{pmlr-v49-garivier16a,
	title        = {Optimal Best Arm Identification with Fixed Confidence},
	author       = {Garivier, Aurélien and Kaufmann, Emilie},
	year         = 2016,
	month        = {23--26 Jun},
	booktitle    = {29th Annual Conference on Learning Theory},
	publisher    = {PMLR},
	address      = {Columbia University, New York, New York, USA},
	series       = {Proceedings of Machine Learning Research},
	volume       = 49,
	pages        = {998--1027},
	url          = {https://proceedings.mlr.press/v49/garivier16a.html},
	editor       = {Feldman, Vitaly and Rakhlin, Alexander and Shamir, Ohad}
}

@article{MAL-068,
	title        = {Introduction to Multi-Armed Bandits},
	author       = {Aleksandrs Slivkins},
	year         = 2019,
	journal      = {Foundations and Trends in Machine Learning},
	volume       = 12,
	number       = {1-2},
	pages        = {1--286},
	doi          = {10.1561/2200000068},
	issn         = {1935-8237},
	url          = {http://dx.doi.org/10.1561/2200000068}
}

@article{kuleshov2014algorithms,
	title        = {Algorithms for multi-armed bandit problems},
	author       = {Kuleshov, Volodymyr and Precup, Doina},
	year         = 2014,
	journal      = {arXiv preprint arXiv:1402.6028}
}

@inproceedings{NIPS2014_d3ea0f33,
	title        = {Combinatorial Pure Exploration of Multi-Armed Bandits},
	author       = {Chen, Shouyuan and Lin, Tian and King, Irwin and Lyu, Michael R. and Chen, Wei},
	year         = 2014,
	booktitle    = {Advances in Neural Information Processing Systems},
	publisher    = {Curran Associates, Inc.},
	volume       = 27,
	url          = {https://proceedings.neurips.cc/paper_files/paper/2014/file/d3ea0f3316d2da934d79b8b344eafee4-Paper.pdf},
	editor       = {Z. Ghahramani and M. Welling and C. Cortes and N. Lawrence and K.Q. Weinberger}
}

@inproceedings{pmlr-v28-chen13a,
	title        = {Combinatorial Multi-Armed Bandit: General Framework and Applications},
	author       = {Chen, Wei and Wang, Yajun and Yuan, Yang},
	year         = 2013,
	month        = {17--19 Jun},
	booktitle    = {Proceedings of the 30th International Conference on Machine Learning},
	publisher    = {PMLR},
	address      = {Atlanta, Georgia, USA},
	series       = {Proceedings of Machine Learning Research},
	volume       = 28,
	pages        = {151--159},
	url          = {https://proceedings.mlr.press/v28/chen13a.html},
	editor       = {Dasgupta, Sanjoy and McAllester, David}
}

@inproceedings{pmlr-v80-wang18a,
	title        = {Thompson Sampling for Combinatorial Semi-Bandits},
	author       = {Wang, Siwei and Chen, Wei},
	year         = 2018,
	month        = {10--15 Jul},
	booktitle    = {Proceedings of the 35th International Conference on Machine Learning},
	publisher    = {PMLR},
	series       = {Proceedings of Machine Learning Research},
	volume       = 80,
	pages        = {5114--5122},
	url          = {https://proceedings.mlr.press/v80/wang18a.html},
	editor       = {Dy, Jennifer and Krause, Andreas}
}

@inproceedings{qin2014contextual,
	title        = {Contextual combinatorial bandit and its application on diversified online recommendation},
	author       = {Qin, Lijing and Chen, Shouyuan and Zhu, Xiaoyan},
	year         = 2014,
	booktitle    = {Proceedings of the 2014 SIAM International Conference on Data Mining},
	pages        = {461--469},
	organization = {SIAM}
}

@inproceedings{NEURIPS2018_207f8801,
	title        = {Contextual Combinatorial Multi-armed Bandits with Volatile Arms and Submodular Reward},
	author       = {Chen, Lixing and Xu, Jie and Lu, Zhuo},
	year         = 2018,
	booktitle    = {Advances in Neural Information Processing Systems},
	publisher    = {Curran Associates, Inc.},
	volume       = 31,
	url          = {https://proceedings.neurips.cc/paper_files/paper/2018/file/207f88018f72237565570f8a9e5ca240-Paper.pdf},
	editor       = {S. Bengio and H. Wallach and H. Larochelle and K. Grauman and N. Cesa-Bianchi and R. Garnett}
}

@inproceedings{li2016contextual,
	title        = {Contextual combinatorial cascading bandits},
	author       = {Li, Shuai and Wang, Baoxiang and Zhang, Shengyu and Chen, Wei},
	year         = 2016,
	booktitle    = {International conference on machine learning},
	pages        = {1245--1253},
	organization = {PMLR}
}

@inproceedings{sankararaman2018combinatorial,
	title        = {Combinatorial semi-bandits with knapsacks},
	author       = {Sankararaman, Karthik Abinav and Slivkins, Aleksandrs},
	year         = 2018,
	booktitle    = {International Conference on Artificial Intelligence and Statistics},
	pages        = {1760--1770},
	organization = {PMLR}
}

@article{li2019combinatorial,
	title        = {Combinatorial sleeping bandits with fairness constraints},
	author       = {Li, Fengjiao and Liu, Jia and Ji, Bo},
	year         = 2019,
	journal      = {IEEE Transactions on Network Science and Engineering},
	publisher    = {IEEE},
	volume       = 7,
	number       = 3,
	pages        = {1799--1813}
}

@inproceedings{NEURIPS2019_5e388103,
	title        = {Combinatorial Bandits with Relative Feedback},
	author       = {Saha, Aadirupa and Gopalan, Aditya},
	year         = 2019,
	booktitle    = {Advances in Neural Information Processing Systems},
	publisher    = {Curran Associates, Inc.},
	volume       = 32,
	url          = {https://proceedings.neurips.cc/paper_files/paper/2019/file/5e388103a391daabe3de1d76a6739ccd-Paper.pdf},
	editor       = {H. Wallach and H. Larochelle and A. Beygelzimer and F. d\textquotesingle Alch\'{e}-Buc and E. Fox and R. Garnett}
}

@article{zuo2025pareto,
	title        = {On Pareto Optimality for the Multinomial Logistic Bandit},
	author       = {Zuo, Jierui and Qin, Hanzhang},
	year         = 2025,
	journal      = {arXiv preprint arXiv:2501.19277}
}

@misc{faruk2025learningpeerinfluenceprobabilities,
      title={Learning Peer Influence Probabilities with Linear Contextual Bandits}, 
      author={Ahmed Sayeed Faruk and Mohammad Shahverdikondori and Elena Zheleva},
      year={2025},
      eprint={2510.19119},
      archivePrefix={arXiv},
      primaryClass={cs.LG},
      url={https://arxiv.org/abs/2510.19119}, 
}

@misc{zhang2025onlineexperimentaldesignestimationregret,
      title={Online Experimental Design With Estimation-Regret Trade-off Under Network Interference}, 
      author={Zhiheng Zhang and Zichen Wang},
      year={2025},
      eprint={2412.03727},
      archivePrefix={arXiv},
      primaryClass={cs.LG},
      url={https://arxiv.org/abs/2412.03727}, 
}

@article{Ye_Zhang_Zhang_Zhang_Zhang_2023,
	title        = {Deep learning based causal inference for large-scale combinatorial experiments: Theory and empirical evidence},
	author       = {Ye, Zikun and Zhang, Zhiqi and Zhang, Dennis and Zhang, Heng and Zhang, Renyu},
	year         = 2023,
	journal      = {SSRN Electronic Journal},
	doi          = {10.2139/ssrn.4375327},
	issn         = {1556-5068},
	url          = {https://papers.ssrn.com/abstract=4375327},
	language     = {en}
}
}

\newpage
\mbox{} 
\newpage

\section*{Checklist}

\begin{enumerate}

  \item For all models and algorithms presented, check if you include:
  \begin{enumerate}
    \item A clear description of the mathematical setting, assumptions, algorithm, and/or model. [Yes]
    \item An analysis of the properties and complexity (time, space, sample size) of any algorithm. [Not Applicable]
    \item (Optional) Anonymized source code, with specification of all dependencies, including external libraries. [Not Applicable]
  \end{enumerate}

  \item For any theoretical claim, check if you include:
  \begin{enumerate}
    \item Statements of the full set of assumptions of all theoretical results. [Yes]
    \item Complete proofs of all theoretical results. [Yes]
    \item Clear explanations of any assumptions. [Yes]     
  \end{enumerate}

  \item For all figures and tables that present empirical results, check if you include:
  \begin{enumerate}
    \item The code, data, and instructions needed to reproduce the main experimental results (either in the supplemental material or as a URL). [Not Applicable]
    \item All the training details (e.g., data splits, hyperparameters, how they were chosen). [Not Applicable]
    \item A clear definition of the specific measure or statistics and error bars (e.g., with respect to the random seed after running experiments multiple times). [Not Applicable]
    \item A description of the computing infrastructure used. (e.g., type of GPUs, internal cluster, or cloud provider). [Not Applicable]
  \end{enumerate}

  \item If you are using existing assets (e.g., code, data, models) or curating/releasing new assets, check if you include:
  \begin{enumerate}
    \item Citations of the creator If your work uses existing assets. [Not Applicable]
    \item The license information of the assets, if applicable. [Not Applicable]
    \item New assets either in the supplemental material or as a URL, if applicable. [Not Applicable]
    \item Information about consent from data providers/curators. [Not Applicable]
    \item Discussion of sensible content if applicable, e.g., personally identifiable information or offensive content. [Not Applicable]
  \end{enumerate}

  \item If you used crowdsourcing or conducted research with human subjects, check if you include:
  \begin{enumerate}
    \item The full text of instructions given to participants and screenshots. [Not Applicable]
    \item Descriptions of potential participant risks, with links to Institutional Review Board (IRB) approvals if applicable. [Not Applicable]
    \item The estimated hourly wage paid to participants and the total amount spent on participant compensation. [Not Applicable]
  \end{enumerate}

\end{enumerate}

\begin{onecolumn}
\appendix
\section{Table of Notations}
\begin{table}[h!]
\centering
\begin{tabular}{ll}
\hline
\textbf{Notation} & \textbf{Description} \\
\hline
$[i]$ & The set$\{1,...,i\}$, given a positive integer $i$\\
$n$ &  Time horizon length \\
$\mathbb{E}[X]$ & Expectation of random variable $X$ \\
$d$ & Total number of base arms \\
$e$ & The base arms \\
$m$ & Numbers of based arms in a super arm $M$ \\
$M(t)$ & The super arm chosen at time $t$\\
$\mathcal{A}$ & The set of base arms \\
$\mathcal{M}$ & The set of super arms \\
$\omega$ & Feedback vector of the model \\
$\boldsymbol{\mu}$ & Expectation of feedback vector\\
$\nu$ & Distribution of Feedback \\
$f(G,\varpi)$ & Mapping $f:2^{\mathcal{A}}\times\mathbb{R}^d\rightarrow\mathbb{R}$ such that $f(G,\varpi)=\sum_{e\in G}\varpi(e)$ \\
$\pi_t(M)$ & Policy probability at time $t$ choosing super arm $M$\\
$\mathcal{H}_t$ & History $(M(1);w_{1},\ldots,M(t);w_{t})$ at time $t$\\
$M^*$ & The super arm with maximun expected reward $\arg\max_{M\in\mathcal{M}}\sum_{e\in M}\mu(e)$\\
$\mathcal{R}_{\nu}(n,\pi)$ & The expected regret with time horizon $n$ and policy $\pi$\\
$\Delta_{M}^{(i,j)}$ & Super arm gap $f(M(\tau_{i}),\boldsymbol{\mu})-f(M(\tau_{j}),\boldsymbol{\mu}),\forall i\neq j\in[|\mathcal{M}|]$\\
$\Delta_{\mu}^{(i,j)}$ & Base arm gap $\mu(i)-\mu(j),\forall i\neq j\in[d]$\\
$\Delta_{M}$ & The gap to best super arm $f(M^*,\boldsymbol{\mu})-f(M,\boldsymbol{\mu})$\\
$\hat{\Delta}_{\mu,t}^{(i,j)}$ & The adaptive base arm gap at time $t$\\
$\hat{\Delta}_{M,t}^{(i,j)}$ & The adaptive super arm gap at time $t$\\
$\mathcal{E}(t,\hat{\Delta}_{t}^{(i,j)})$ & The estimation quality $\mathbb{E}[|\Delta^{(i,j)}-\hat{\Delta}_t^{(i,j)}|]$ of estimator $\hat{\Delta}_{t}^{(i,j)}$ at time $t$\\
$\mathcal{P}_f$ & The Pareto Frontier defined in \ref{def:paretof}\\
$f(n)\preceq g(n)$ & $f(n)/g(n)$ bounded above and below by constant independent of $n$\\
$\boldsymbol{\theta}_M$ & Vectorized $M$ as $[\mathbf{I}\{1\in M\},...,\mathbf{I}\{d\in M\}]^\intercal$\\
$\mathrm{KL}(p,q)$ & KL-divergence of distribution $p,q$\\
$\lambda_{\min}$ & Smallest nonzero eigenvalue of $\mathbb{E}[\boldsymbol{\theta}_M\boldsymbol{\theta}_M^\top]$ with $M$ uniformly distributed\\
$\rho_0$ & Distribution on $\mathcal{A}$ induced by uniform $\mathcal{M}$\\
$\rho_{\min}$ & $\operatorname*{min}_{e\in\mathcal{A}}m\rho_{e}^{0}$\\
$C$ & Acuteness of super arms set $C= \lambda_{\min}m^{-\frac{3}{2}}$\\
$\gamma$ & Mixing parameter $\gamma=[\sqrt{m \log \rho_{\min}^{-1}}+ \sqrt{C \big(C m^{2} d + m\big) n}]^{-1} \sqrt{m \log \rho_{\min}^{-1}}$\\
$\eta$ & Exponential control parameter $\eta=\gamma C$\\
\hline
\end{tabular}
\caption{Summary of notations used in this paper.}
\end{table}

\section{Computational Results}
\label{app:compres}
We define an algorithm as computationally efficient if it can be implemented efficiently whenever the offline version of the problem has an efficient solution (i.e., implementable up to polynomial factors). We now consider the computational problem of our algorithms \texttt{MixCombKL} and \texttt{MixCombUCB}. In this section, we demonstrate that our algorithms can achieve computational feasibility while maintaining a decent regret upper bound, thus retaining Pareto optimality even with limited computational resources.

\subsection{Discussion on Full-Bandit Feedback}
\label{sec:Disfullband}
We first present Proposition \ref{Thm:compeffi}, which shows the computational efficiency of offline \texttt{MixCombKL}:

\begin{proposition}
\label{Thm:compeffi}
\texttt{MixCombKL} is computationally efficient for any stochastic combinatorial bandits where the offline optimization oracle $\arg\min_{p\in\mathcal{Q}}\mathrm{KL}(p,q)$ can be implemented efficiently for any probability distribution $q:q(i)\in\mathbb{R}^{+},\forall i\in[d],\sum_{j=1}^dq(j)=1$.
\end{proposition}
\begin{proof}
At each iteration \(t\), the vector \(m q_{t-1}'\) can be expressed as a convex combination of at most \(d+1\) arms. By Carathéodory’s theorem, such a representation always exists and can be found within \(\mathcal{O}(d^4)\) time. Consequently, the probability vector \(p_{t-1}\) has at most \(d+1\) non-zero entries. Constructing the matrix \(\Sigma_{t-1}\) then requires \(\mathcal{O}(d^2)\) time, since it is based on the weighted sum of at most \(d+1\) rank-1 matrices \(\boldsymbol{\theta}_M \boldsymbol{\theta}_M^\top\), each computed in \(\mathcal{O}(d)\) time. Computing the pseudo-inverse of \(\Sigma_{t-1}\) incurs a computational cost of \(\mathcal{O}(d^3)\). Aside from the offline optimization oracle, all other operations in the \texttt{MixCombKL} algorithm run in polynomial time with respect to \(m\) and \(d\), establishing that \texttt{MixCombKL} is computationally efficient.    
\end{proof}

The above proposition demonstrates that the offline version of \texttt{MixCombKL} can be implemented up to polynomial factors. For the online optimization step, one might argue that it may not be possible to compute the projection step exactly, as we are solving an optimization problem on a continuous space. Thus, we are interested in the case where the projection step can be solved up to accuracy $\epsilon_t$ in round $t$, namely, we find $q_t$ such that $\mathrm{KL}(q_t,\tilde{q}_t)-\min_{p\in\Xi}\mathrm{KL}(p,\tilde{q}_t)\leq\epsilon_t$.
Proposition \ref{Thm:regapprox} shows that for $\epsilon_t=\mathcal{O}(t^{-2}\log^{-3}(t))$, the approximate projection gives the same regret as when the projection is computed exactly, and its computational complexity will be revealed later.

\begin{proposition}
\label{Thm:regapprox}
If the projection step of \texttt{MixCombKL} is solved up to accuracy $\epsilon_t=\mathcal{O}(t^{-2}\log^{-3}(t))$, the regret has
\begin{align*}
\mathcal{R}_{\nu}(n,\pi)\leq2\sqrt{2m^{3}n\left(d+\frac{m^{1/2}}{\lambda_{\min}}\right)\log\rho_{\min}^{-1}}+\frac{mn^{1-\alpha}}{2(1-\alpha)}+\frac{2m^{5/2}}{\lambda_{\min}}\log\rho_{\min}^{-1}.
\end{align*}
\end{proposition}

The proof of the above proposition is shown in Section \ref{sec:proof-regapprox}. Our focus is to verify the computational efficiency of our algorithm when solving the KL-divergence projection step up to a certain accuracy. Suppose the convex set \(\mathrm{Co}(\mathcal{\boldsymbol{\theta}})\) is characterized by \(c\) linear equalities and \(r\) linear inequalities. The projection step corresponds to solving a convex optimization problem to an accuracy of \(\epsilon_t = \mathcal{O}(t^{-2} \log^{-3}(t))\). We employ the Interior-Point Method (specifically, the barrier method) to solve this. 

The total number of Newton iterations required to achieve accuracy \(\epsilon_t\) is \(\mathcal{O}(\sqrt{r} \log(r/\epsilon_t))\), and each iteration incurs a computational cost of \(\mathcal{O}((d + c)^3)\) \cite[Ch. 10 and 11]{Boyd_Vandenberghe_2004}. Therefore, the overall complexity of the projection step at iteration \(t\) is \(\mathcal{O}(\sqrt{r} (c + d)^3 \log(r / \epsilon_t))\). By substituting \(\epsilon_t = \mathcal{O}(t^{-2} \log^{-3}(t))\), we see that the cumulative cost across \(n\) iterations becomes \(\mathcal{O}(\sqrt{r} (c + d)^3 n \log n)\). Hence, when \(\mathrm{Co}(\mathcal{\boldsymbol{\theta}})\) is described by polynomially many linear constraints, the projection step can be implemented within polynomial factors.

Combining the results above shows that \texttt{MixCombKL} can be implemented within polynomial factors if the KL divergence is solved up to accuracy \(\epsilon_t = \mathcal{O}(t^{-2}\log^{-3}(t))\). Proposition \ref{Thm:regapprox} further shows that the regret remains the same as when the projection is computed exactly, and since the accuracy of arm estimation is unaffected by KL projection errors, it follows from the analysis in Section \ref{sec:Paretocondition} that \texttt{MixCombKL} attains Pareto optimality even under limited computational resources.

\subsection{Discussion on Semi-Bandit Feedback}
Similarly, we have the following result for computational efficiency in the semi-bandit context:

\begin{proposition}
\label{Thm:compeffisemi}
\texttt{MixCombUCB} is computationally efficient in any stochastic semi-bandit where the offline optimization oracle \(\arg\max_{M\in\mathcal{M}}f(M,w)\) can be implemented efficiently for any \(w\in(\mathbb{R}^{+})^{d}\).    
\end{proposition}

\begin{proof}
In each step \(t\), \texttt{MixCombUCB} calls the oracle once, and all of its remaining operations are polynomial in \(m\) and \(d\). Therefore, \texttt{MixCombUCB} is guaranteed to be computationally efficient when the oracle is computationally efficient.
\end{proof}

We now see that the offline version of \texttt{MixCombUCB} can be implemented up to polynomial factors. Unlike the full-bandit feedback case in Section \ref{sec:Disfullband}, we can assume that the offline optimization oracle is exactly solved, as \(\arg\max_{M\in\mathcal{M}}f(M,w)\) operates on a discrete space \(\mathcal{M}\), thus keeping the regret bound and estimation error unchanged. Following the analysis of Section \ref{sec:Paretocondition}, \texttt{MixCombUCB} can achieve Pareto optimality under polynomial computational resources.

\subsection{Computational Example on Problem-Dependent Constants}
In this subsection, we discuss about the computational results of certain problem dependent constants with specific super arms set structure.

For example, when the super-arms set $\mathcal{M}$ consist of all subsets up to size 
$m$, we can compute that $\rho_{\min}=\frac md,\lambda_{\min}=\frac{m(d-m)}{d(d-1)}$, and when $\mathcal{M}$ is the set of perfect matchings, where $d=m^2$ and $|\mathcal{M}|=m!$, we can compute that $\rho_{\min}=\frac 1m,\lambda_{\min}=\frac{1}{m-1}$.

These concrete examples echos with the observation in \cite{Csiszar2004} that for most classes of $\mathcal{M}$, there is $\rho _{\mathrm{min}}^{- 1}= \mathcal{O}(\text{poly}(d))$ and $m(d\lambda_{\min})^{-1}=\mathcal{O}(1)$. Thus, the \texttt{MixCombKL} algorithm have a regret of $\mathcal{O}(\sqrt{m^3dn\log(d/m)}+mn^{1-\alpha})$ on these problem classes.
\section{Analysis of Algorithm \ref{alg:mixcombkl}}

\subsection{Proof of Theorem \ref{Thm:delta1}.}
We first prove a simple result:
\begin{lemma}
\label{Lm:projM}
For all $x \in \mathbb{R}^d$, we have $\Sigma_{t-1}^+ \Sigma_{t-1} x = \overline{x}$, where $\overline{x}$ is the orthogonal projection of $x$ onto $\mathrm{span}(\boldsymbol{\theta})$, the linear space spanned by $\boldsymbol{\theta}_M$.    
\end{lemma}
\begin{proof}
Note that for all $y\in\mathbb{R}^d$, if $\Sigma_{t-1}y=0$, then we have
\begin{equation}
\label{Eq:ySigmay}
y^\top\Sigma_{t-1}y=\mathbb{E}\left[y^\top \boldsymbol{\theta}_{M(t)}\boldsymbol{\theta}_{M(t)}^\top y\right]=\mathbb{E}\left[(y^\top M(t))^2\right]=0,    
\end{equation}
(i) If $\boldsymbol{\theta}_{M(t)}$ has law $p_{t-1}$ such that $\sum_{M\in\mathcal{M}}\boldsymbol{\theta}_{M(t)}(i)p_{t-1}(M)=q_{t-1}^{\prime}(i)$, $\forall i\in[d]$ and $q_{t-1}^{\prime}=(1-\gamma)q_{t-1}+\gamma \mu ^{0}$. By definition of $\mu ^{0}$, each $M(t)\in\mathcal{M}$ has a positive probability. 
(ii) If $M(t)$ is uniformly distributed, of course each $M(t)\in\mathcal{M}$ has a positive probability. 
Thus we conclude that $\forall t\in[n]$, each $M(t)\in\mathcal{M}$ has a positive probability. Hence, by \ref{Eq:ySigmay}, $y^{\top}\boldsymbol{\theta}_{M}=0$ for all $M \in \mathcal{M}.$ In particular, we see that the linear application $\Sigma_{t-1}$ restricted to $\mathrm{span}(\boldsymbol{\theta})$ is invertible and is zero on $\mathrm{span}(\boldsymbol{\theta})^{\perp}$, hence we have $\Sigma_{t-1}^{+}\Sigma_{t-1}x = \overline{x}$.
\end{proof}
We consider fillteration $\mathcal{F}_n=\sigma(w_1,w_2,...,w_n)$ , we define the follow martingales 
$$r_i(\tau)=\sum_{t=1}^{\tau}\sum_{j=1}^{d}2\tau^{\alpha}\mathbb{I}\{U_t=1\}\mathbb{I}\{i\in M\}\mathbb{I}\{j\in M\}w_t(j)-\tau\sum_{t=1}^{\tau}\sum_{j=1}^{d}\mathbb{E}[\mathbb{I}\{i\in M\}\mathbb{I}\{j\in M\}]\mathbb{E}[w_t(j)],\: \forall i\in[d], \tau\in[n].$$
It's easy to see that
\begin{equation*}
\begin{aligned}
|r_i(t)-r_i(t-1)|&\leq|\sum_{j=1}^{d}2t^{\alpha}\mathbb{I}\{U_t=1\}\mathbb{I}\{i\in M\}\mathbb{I}\{j\in M\}w_t(j)|+|\sum_{j=1}^{d}\mathbb{E}[\mathbb{I}\{i\in M\}\mathbb{I}\{j\in M\}]\mathbb{E}[w_t(j)]|
\\&\leq\sum_{j=1}^d2t^{\alpha}|\mathbb{I}\{i\in M\}\mathbb{I}\{j\in M\}w_t(j)|
+\sum_{j=1}^d\mathbb{E}|\mathbb{I}\{i\in M\}\mathbb{I}\{j\in M\}|\mathbb{E}|w_t(j)|\\&
\leq\sum_{j=1}^d2t^{\alpha}|\mathbb{I}\{j\in M\}|
+\sum_{i=1}^d\mathbb{E}|\mathbb{I}\{j\in M\}|\leq3mt^{\alpha}\leq3mt^{\frac{\alpha+1}{2}}
\end{aligned}    
\end{equation*}
the third equality holds as $|\mathbb{I}\{e\in M\}|,|w_t(e)|\leq1$, $\forall e\in\mathcal{A}$.
we denote $s_i=\sum_{j=1}^{d}\mathbb{E}[\mathbb{I}\{i\in M\}\mathbb{I}\{j\in M\}]\mathbb{E}[w_t(j)]$, for all $i\in[d]$, We write $V_{i,t}=\sum_{t=1}^{\tau}\mathbb{E}[(r_i(t)-r_i(t-1))^2|\mathcal{F}_{t-1}]$ and we have:

\begin{equation*}
\begin{aligned}
\sum_{t=1}^{\tau}\mathbb{E}[(r_i(t)-r_i(t-1))^2|\mathcal{F}_{t-1}]&=\sum_{t=1}^{\tau}\mathbb{E}[(2t^{\alpha}\mathbb{I}\{U_t=1\}\sum_{j=1}^{d}\mathbb{I}\{i\in M\}\mathbb{I}\{j\in M\}w_t(j))^2|\mathcal{F}_{t-1}]-\tau s_i^2
\\&\leq\sum_{t=1}^{\tau}\mathbb{E}[(2t^{\alpha}\mathbb{I}\{U_t=1\}\sum_{j=1}^{d}|\mathbb{I}\{i\in M\}\mathbb{I}\{j\in M\}w_t(j)|)^2|\mathcal{F}_{t-1}]\\&
\leq\sum_{t=1}^{\tau}2t^{\alpha}\mathbb{E}[(\sum_{j=1}^{d}|\mathbb{I}\{j\in M\}|)^2|\mathcal{F}_{t-1}]\leq \frac{2m^2(\tau+1)^{\alpha+1}}{\alpha+1}
\end{aligned}    
\end{equation*}
Now with $\delta\in(0,1]$, we have determistic upper bound $\hat{V}_t=\frac{4m^2(\tau+1)^{2\alpha+1}}{2\alpha+1}\geq V_{i,t}, \forall i\in[d]$
such that $$\hat{V}_t\leq \frac{2m^2(\tau+1)^{\alpha+1}}{\alpha+1}\: \vee\frac{9m^2\tau^{\alpha+1}\ln(2/\delta)}{e-2}\leq\frac{9m^2\tau^{\alpha+1}\ln(2/\delta)}{e-2}.$$
by Bernstein’s Inequality, we know that with possibility more than $1-\delta$, we have
\begin{equation*}
|r_i(t)|\leq\sqrt{(e-2)\frac{9m^2t^{\alpha+1}(\ln(2/\delta))^2}{e-2}}=3m\ln(\frac{2}{\delta})t^{\frac{\alpha+1}{2}},\:\forall t\in[n].    
\end{equation*}
Applying union bound, we can see with probability more than $1-\delta$, there is $|r_i(t)|\leq 3m\ln(\frac{2}{\delta})t^{\frac{\alpha+1}{2}}$, $\forall t\in[n],i\in[d].$ we write $\boldsymbol{r}(t)=(r_1(t),...,r_d(t))^{\top}$ and we can find that 
$$\sum_{t=1}^{\tau}(\tilde{w}_t-\mathbb{E}[\tilde{w}_t])=\Sigma_{t-1}^{+}\boldsymbol{r}(\tau),\forall\tau\in[n].$$
By our update rule we know that all $\Sigma_{t-1}^{+}$ in the above equality is the pseudo-inverse of $\Sigma_{t-1}=\mathbb{E}\left[\boldsymbol{\theta}_M \boldsymbol{\theta}_M^{\top}\right]$ where $M$ has uniform distribution on $\mathcal{M}$. Now we can find everytime we update ${R}_{t}(M (\tau_{k}))$, there is
\begin{equation*}
||\sum_{t=1}^{\tau}(\tilde{w}_t-\mathbb{E}[\tilde{w}_t])||_2 \leq||\Sigma_{t-1}^{+}||_{op}||\boldsymbol{r}(t)||_2\leq\frac{1}{\lambda_{\min}}||\boldsymbol{r}(t)||_2\leq\frac{3m\sqrt{dt^{\alpha+1}}}{\lambda_{\min}}\ln(\frac{2d}{\delta}).    
\end{equation*}
We use $E_t(Z)=E(Z|w_t)$ for any random variable $Z$ to denote conditional probability on $w_t$ and we have 
\begin{equation}
\label{Eq:EX(t)}
\mathbb{E}_t\left[\tilde{w}_t\right]=\mathbb{E}_t\left[f(M(t),w_t)\Sigma_{t-1}\boldsymbol{\theta}_{M(t)}\right]=\mathbb{E}_t\left[\Sigma_{t-1}^+\boldsymbol{\theta}_{M(t)}\boldsymbol{\theta}_{M(t)}^{\top}w_t\right]=\Sigma_{t-1}^+\Sigma_{t-1}w_t=\overline{w_t},    
\end{equation}
where the last equality follows from Lemma \ref{Lm:projM} and $\overline{w_t}$ is the orthogonal projection of $w_t$ onto $\mathrm{span}(\boldsymbol{\theta})$, 
therefore, we have $\mathbb{E}(\overline{w_t})\in\mathrm{span}(\boldsymbol{\theta}),\forall t\in[n]$. We can derive 
$$\boldsymbol{\theta}_{M(\tau_{k})}^{\top}\mathbb{E}(\overline{w_t})=\mathbb{E}(\boldsymbol{\theta}_{M(\tau_{k})}^{\top}\overline{w_t})=\mathbb{E}(\boldsymbol{\theta}_{M(\tau_{k})}^{\top}w_t)=\boldsymbol{\theta}_{M(\tau_{k})}^{\top}\boldsymbol{\mu}=f(M(\tau_k),\boldsymbol{\mu}),\forall k\in[|\mathcal{M}|]$$
In particular, for any $mq^{\prime}\in \mathrm{Co}({\boldsymbol{\theta}})$,we have $\mathbb{E}_t\left[mq'^\top\tilde{w}_t\right]=mq'^\top\overline{w_t}=mq'^\top w_t.$

We now can find that
\begin{equation}
\label{Eq:RMtauBound}
\begin{aligned}
R_t(M(\tau_k))-tf(M(\tau_k),\boldsymbol{\mu})&=\boldsymbol{\theta}_{M(\tau_{k})}^{\top}\sum_{t=1}^{\tau}(\tilde{w}_t-\mathbb{E}[\tilde{w}_t])=\hat{R}_t(M(\tau_k))\\&\leq||\boldsymbol{\theta}_{M(\tau_{k})}||_2 ||\sum_{t=1}^{\tau}(\tilde{w}_t-\mathbb{E}[\tilde{w}_t])||_2\leq\frac{3\sqrt{m^3dt^{\alpha+1}}}{\lambda_{\min}}\ln\left(\frac{2d}{\delta}\right).     
\end{aligned} 
\end{equation}
\ref{Eq:RMtauBound} shows that
\begin{equation*}
\begin{aligned}
|\hat{\Delta}_n^{(i,j)}-\Delta^{(i,j)}|&=\mathbb{E}[|(\boldsymbol{\theta}_{M(\tau_{i})}^{\top}\hat{\boldsymbol{\mu}}-\boldsymbol{\theta}_{M(\tau_{i})}^{\top}\boldsymbol{\mu})-
(\boldsymbol{\theta}_{M(\tau_{j})})^{\top}\hat{\boldsymbol{\mu}}-\boldsymbol{\theta}_{M(\tau_{j})}^{\top}\boldsymbol{\mu})|]
\\&\leq\mathbb{E}
[\frac{|\hat{R}_t(M(\tau_i))|+|\hat{R}_t(M(\tau_j))|}{n}]\leq\frac{6}{\lambda_{\min}}\sqrt{\frac{m^3d}{n^{1-\alpha}}}\ln\left(\frac{2d}{\delta}\right).
\end{aligned}
\end{equation*}
Taking $\delta=\frac{1}{n^2}$ and we can achieve that $\mathcal{E}(n,\hat{\Delta}_{n}^{(i,j)})=\mathcal{O}(\sqrt{n^{\alpha-1}}).$

\subsection{Proof of Theorem \ref{Thm:R1}.}
For simplicity, we consider two class of time period denoted as $V_n=\{t:U_t=0,t=1,...,n\}$ and $W_n=\{t:U_t=1,t=1,...,n\}$, assume we have $mq^*=M^*$ as the optimal arm, i.e. $q^*(i)=\frac{1}{m}$ if $M_i^*=1$, we can write our regret as follows:
\begin{equation}
\label{Eq:Rnpi}
\begin{aligned}
\mathcal{R}_{\nu}(n,\pi)&=\mathbb{E}\Big[\sum_{t=1}^nmq^{\star\top}\tilde{w}_t-\sum_{t=1}^nmq_{t-1}^{\prime\top}\tilde{w}_t\Big]\\&=\mathbb{E}\Big[\sum_{t\in V_n}mq^{\star\top}\tilde{w}_t-\sum_{t\in V_n}mq_{t-1}^{\prime\top}\tilde{w}_t\Big]+
\mathbb{E}\Big[\sum_{t\in W_n}mq^{\star\top}\tilde{w}_t-\sum_{t\in W_n}mq_{t-1}^{\prime\top}\tilde{w}_t\Big]\\&\leq
\mathbb{E}\Big[\sum_{t\in V_n}mq^{\star\top}\tilde{w}_t-\sum_{t\in V_n}mq_{t-1}^{\prime\top}\tilde{w}_t\Big]+m\mathbb{E}|W_n|\leq\mathcal{R}_{\nu}^{\mathrm{Comb}}(n,\pi)+\frac{mn^{1-\alpha}}{2(1-\alpha)}.
\end{aligned}    
\end{equation}
the first equality holds because of \ref{Eq:EX(t)} and we use $\mathcal{R}_{\nu}^{\mathrm{Comb}}(n,\pi)$ to denote the regret caused by sample throught $p_{t-1}$. We first proceed with the following lemma:
\begin{lemma}
\label{Le:qtKL}
We have for any $\eta\leq\frac{\gamma\lambda_{\min}}{m^{3/2}}$ and any $q\in\mathcal{P}$,
$$\begin{aligned}\sum_{t\in V_n}q^\top\tilde{w}_t-\sum_{t\in V_n}q_{t-1}^\top\tilde{w}_t\leq\eta\sum_{t\in V_n}q_{t-1}^\top\tilde{w}^2_t+\frac{\mathrm{KL}(q,q_0)}{\eta},\end{aligned}$$
where $\tilde{w}^{2}_t$ is the vector that is the coordinate-wise square of $\tilde{w}_t$.
\end{lemma}
We have:
$$\begin{aligned}\mathbb{E}_{t}\left[q_{t-1}^{\top}\tilde{w}^{2}_t\right]&=\sum_{i\in[d]}q_{t-1}(i)\mathbb{E}_{t}\left[\tilde{w}^{2}_t(i)\right]=\sum_{i\in[d]}\frac{q_{t-1}^{\prime}(i)-\gamma\mu^{0}(i)}{1-\gamma}\mathbb{E}_{t}\left[\tilde{w}^{2}_t(i)\right]\\&\leq\frac{1}{m(1-\gamma)}\sum_{i\in[d]}mq_{t-1}^{\prime}(i)\mathbb{E}_{t}\left[\tilde{w}^{2}_t(i)\right]=\frac{1}{m(1-\gamma)}\mathbb{E}_{t}\Big[\sum_{i\in[d]}\boldsymbol{\theta}_{\tilde{M}(t)}(i)\tilde{w}^{2}_t(i)\Big],\end{aligned}$$
where $\tilde{M}(t)$ is a random arm with the same law as $M(t)$ and independent of $M(t)$. Note that $\boldsymbol{\theta}_{\tilde{M}(t)}^2 = \boldsymbol{\theta}_{\tilde{M}(t)}$, so that we have
$$\begin{aligned}\mathbb{E}_t\bigg[\sum_{i\in[d]}\boldsymbol{\theta}_{\tilde{M}(t)}(i)\tilde{w}_t^2(i)\bigg] &= \mathbb{E}_t\bigg[w_t^{\top}\boldsymbol{\theta}_{M(t)}\boldsymbol{\theta}_{M(t)}^{\top}\Sigma_{t-1}^+\boldsymbol{\theta}_{\tilde{M}(t)}\boldsymbol{\theta}_{\tilde{M}(t)}^{\top}\Sigma_{t-1}^+\boldsymbol{\theta}_{M(t)}\boldsymbol{\theta}_{M(t)}^{\top}w_t\bigg]
\\&\leq m^2\mathbb{E}_t[\boldsymbol{\theta}_{M(t)}^{\top}\Sigma_{t-1}^+\boldsymbol{\theta}_{M(t)}],\end{aligned}$$
where we used the bound $\boldsymbol{\theta}_{M(t)}^\top w_t \leq m$. By [\cite{CESABIANCHI20121404} Lemma 15], $\mathbb{E}_t[\boldsymbol{\theta}_{M(t)}^\top \Sigma_{t-1}^+ \boldsymbol{\theta}_{M(t)}] \leq d$, so that we have:
$$\mathbb{E}_t \left[ q_{t-1}^\top \tilde{w}^2_t \right] \leq \frac{md}{1-\gamma}.$$
Observe that
$$\begin{aligned}\mathbb{E}_{t}\left[q^{\star\top}\tilde{w}_t-q_{t-1}^{\prime\top}\tilde{w}_t\right]&=\mathbb{E}_t\left[q^{\star\top}\tilde{w}_t-(1-\gamma)q_{t-1}^\top\tilde{w}_t-\gamma\mu^{0\top}\tilde{w}_t\right]\\&=\mathbb{E}_t\left[q^{\star\top}\tilde{w}_t-q_{t-1}^\top\tilde{w}_t\right]+\gamma q_{t-1}^\top X(t)-\gamma\mu^{0\top}X(t)\\&\leq\mathbb{E}_t\left[q^{\star\top}\tilde{w}_t-q_{t-1}^\top\tilde{w}_t\right]+\gamma q_{t-1}^\top X(t)\\&\leq\mathbb{E}_t\left[q^{\star\top}\tilde{w}_t-q_{t-1}^\top\tilde{w}_t\right]+\gamma.\end{aligned}$$

Using Lemma  and the above bounds, we get with $mq^{\star}$ the optimal arm, i.e. $q^{\star}(i)=\frac1m$ iff $M_i^{\star}=1$,
$$\begin{aligned}\mathcal{R}_{\nu}^{\mathrm{Comb}}(n,\pi)&=\mathbb{E}\Big[\sum_{t\in V_n}mq^{\star\top}\tilde{w}_t-\sum_{t\in V_n}mq_{t-1}^{\prime\top}\tilde{w}_t\Big]\\&\leq\mathbb{E}\Big[\sum_{t\in V_n}mq^{\star\top}\tilde{w}_t-\sum_{t\in V_n}mq_{t-1}^{\top}\tilde{w}_t\Big]+mn\gamma\\&\leq\frac{\eta m^{2}dn}{1-\gamma}+\frac{m\log\rho_{\min}^{-1}}{\eta}+m\gamma n,\end{aligned}$$
Since $$\mathrm{KL}(q^{*},q_{0})=-\frac{1}{m}\sum_{i\in M^{*}}\log m\rho_{i}^{0}\leq\log\rho_{\min}^{-1}.$$

Choosing $\eta=\gamma C$ with $C=\frac{\lambda_{\min}}{m^{3/2}}$ gives
$$\begin{aligned}\mathcal{R}_{\nu}^{\mathrm{Comb}}(n,\pi)&\leq\frac{\gamma Cm^{2}dn}{1-\gamma}+\frac{m\log\rho_{\min}^{-1}}{\gamma C}+m\gamma n\\&=\frac{Cm^{2}d+m-m\gamma}{1-\gamma}\gamma n+\frac{m\log\rho_{\min}^{-1}}{\gamma C}\\&\leq\frac{(Cm^{2}d+m)\gamma n}{1-\gamma}+\frac{m\log\rho_{\min}^{-1}}{\gamma C}.\end{aligned}$$
Set $\gamma=\frac{\sqrt{m\log\rho_{\min}^{-1}}}{\sqrt{m\log\rho_{\min}^{-1}}+\sqrt{C(Cm^{2}d+m)n}}$ and by \ref{Eq:Rnpi}, we can see that
$$\forall n\geq1:\mathcal{R}_{\nu}(n,\pi)\leq2\sqrt{m^{3}n\left(d+\frac{m^{1/2}}{\lambda_{\min}}\right)\log\rho_{\min}^{-1}}+\frac{mn^{1-\alpha}}{2(1-\alpha)}+\frac{m^{5/2}}{\lambda_{\min}}\log\rho_{\min}^{-1}.$$

\subsection{Proof of Proposition \ref{Thm:regapprox}}
\label{sec:proof-regapprox}
We first introduce a simple result from Lemma 5 of \cite{NIPS2015_0ce2ffd2}:
\begin{lemma}
\label{Le:KLconvex}
The KL-divergence $z\mapsto$KL$( z, q)$ is 1-strongly convex with respect to the $\| \cdot \| _1$ norm.    
\end{lemma}

Recall that $u_{t}=\mathbb{I}\{U_t=1\}q_{t-1}+\mathbb{I}\{U_t=0\}\arg\operatorname*{min}_{p\in\mathcal{Q}}\mathrm{KL}(p,\tilde{q}_{t})$ and that $q_{t}$ is an $\epsilon_{t}$-optimal solution for the projection step, that is

$$\mathrm{KL}(u_t,\tilde{q}_t)\geq\mathrm{KL}(q_t,\tilde{q}_t)-\epsilon_t.$$

By Lemma \ref{Le:KLconvex} we have

$$\mathrm K\mathrm L(q_t,\bar q_t)-\mathrm K\mathrm L(u_t,\bar q_t)\geq(q_t-u_t)^\top\nabla\mathrm K\mathrm L(u_t,\bar q_t)+\frac{1}{2}\|q_t-u_t\|_1^2\geq\frac{1}{2}\|q_t-u_t\|_1^2,$$
where we used $\left ( q_{t}- u_{t}\right ) ^{\top }\nabla$KL$( u_{t}, \tilde{q} _{t}) \geq 0$ due to first- order optimality condition for $u_{t}.$ 
Hence, $\mathrm{KL}( q_{t}, \tilde{q} _{t}) -\mathrm{KL}( u_{t}, \tilde{q} _{t}) \leq \epsilon _{t}$ implies that $\|q_{t}-u_{t}\|_{\infty}\leq\|q_{t}-u_{t}\|_{1}\leq\sqrt{2\epsilon_{t}}.$
Consider $q^\star$,the distribution over $\mathcal{P}$ for the optimal arm, i.e. $q^\star(i)=\frac1m$ iff $M_i^\star=1.$ Recall that
from proof of Lemma \ref{Le:qtKL} for $q=q^{\star}$ we have
\begin{equation}
\label{Eq:KLqun}
\mathrm{KL}(q^{\star},\bar{q}_{t})-\mathrm{KL}(q^{\star},q_{t-1})\leq\eta q_{t-1}^{\top}\overline{w}_t-\eta q^{\star}{}^{\top}\overline{w}_t+\eta^{2}q_{t-1}^{\top}\overline{w}^2_t.
\end{equation}
Generalized Pythagorean inequality (see Theorem 3.1 in \cite{Csiszar2004}) basically says that if $U_t=0$ then
\begin{equation}
\label{Eq:Pyineq}
\mathrm{KL}(q^{\star},\tilde{q}_{t})\geq\mathrm{KL}(q^{\star},u_{t})+\mathrm{KL}(u_{t},\tilde{q}_{t}).
\end{equation}

Let $\underline{q}_{t}=\operatorname*{min}_{i\in[d]}q_{t}(i).$ Observe that for all $t\geq1$:

$$\begin{aligned}\mathrm{KL}(q^{\star},u_{t})&=\sum_{i\in[d]}q^{\star}(i)\log\frac{q^{\star}(i)}{u_{t}(i)}=-\frac{1}{m}\sum_{i\in M^{\star}}\log mu_{t}(i)\\&\geq-\frac{1}{m}\sum_{i\in M^{\star}}\log m(q_{t}(i)+\sqrt{2\epsilon_{t}})\geq-\frac{1}{m}\sum_{i\in M^{\star}}\Big(\log mq_{t}(i)+\frac{\sqrt{2\epsilon_{t}}}{\underline{q}_{t}}\Big)\\&\geq-\frac{\sqrt{2\epsilon_{t}}}{q_{t}}-\frac{1}{m}\sum_{i\in M^{\star}}\log mq_{t}(i)=-\frac{\sqrt{2\epsilon_{t}}}{\underline{q}_{t}}+\mathrm{KL}(q^{\star},q_{t}),\end{aligned}$$

Plugging this into \ref{Eq:Pyineq}, we get that for all $t$ that satisfy $U_t=0$, there is: 

$$\mathrm{KL}(q^{\star},\tilde{q}_{t})\geq\mathrm{KL}(q^{\star},q_{t})-\frac{\sqrt{2\epsilon_{t}}}{\underline{q}_{t}}+\mathrm{KL}(u_{t},\tilde{q}_{t})\geq\mathrm{KL}(q^{\star},q_{t})-\frac{\sqrt{2\epsilon_{t}}}{\underline{q}_{t}}.$$

Putting this together with \ref{Eq:KLqun} yields that for all $t$ that satisfy $U_t=0$, there is:

$$\mathrm{KL}(q^{\star},q_{t})-\mathrm{KL}(q^{\star},q_{t-1})\leq\eta q_{t-1}^{\top}\overline{w}_t-\eta q^{\star\top}\overline{w}_t+\eta^{2}q_{t-1}^{\top}\overline{w}^2_t+\frac{\sqrt{2\epsilon_{t}}}{\underline{q}_{t}}.$$

Finally, notice that $q_t=q_{t-1}$ for all $t$ that have $U_t=0$ and summing over $t$ gives:

$$\sum_{t=1}^{n}\left(q^{\star\top}\tilde{w}_t-q_{t-1}^{\top}\tilde{w}_t\right)\leq\eta\sum_{t=1}^{n}q_{t-1}^{\top}\tilde{w}^{2}_t+\frac{\mathrm{KL}(q^{\star},q_{0})}{\eta}+\frac{1}{\eta}\sum_{t=1}^{n}\frac{\sqrt{2\epsilon_{t}}}{\underline{q}_{t}}.$$

Defining

$$\epsilon_t=\dfrac{\left(\underline{q}_t\log\rho_{\min}^{-1}\right)^2}{32t^2\log^3(t)},\:\forall t\ge1,$$

and recalling that KL$(q^\star,q_0)\leq\log\rho_{\min}^{-1}$,we get

$$\begin{aligned}\sum_{t=1}^{n}\left(q^{\star\top}\tilde{w}_t-q_{t-1}^{\top}\tilde{w}_t\right)&\leq\eta\sum_{t=1}^{n}q_{t-1}^{\top}\tilde{w}^{2}_t+\frac{\log\rho_{\min}^{-1}}{\eta}+\frac{\log\rho_{\min}^{-1}}{\eta}\sum_{t=1}^{T}\sqrt{\frac{2}{32n^{2}\log^{3}(t+1)}}\\&\leq\eta\sum_{t=1}^{n}q_{t-1}^{\top}\tilde{w}^{2}_t+\frac{2\log\rho_{\min}^{-1}}{\eta},\end{aligned}$$

where we used the fact $\sum_{t\geq1}t^{-1}(\log(t+1))^{-3/2}\leq4.$ We remark that by the properties of KL-divergence and since $q^{\prime}_{t-1}\geq\gamma\rho^0>0$, we have $\underline q_t>0$ at every round $t$, so that $\epsilon_t>0$ at every round $t.$

Using the above result and following the same lines as in the proof of Theorem \ref{Thm:R1}, we have

$$\mathcal{R}_{\nu}^{\mathrm{Comb}}(n,\pi)\leq\frac{\eta m^{2}dn}{1-\gamma}+\frac{2m\log\rho_{\min}^{-1}}{\eta}+m\gamma n$$

Choosing $\eta=\gamma C$ with $C=\frac{\lambda_{\min}}{m^{3/2}}$ gives

$$\mathcal{R}_{\nu}^{\mathrm{Comb}}(n,\pi)\leq\frac{(Cm^{2}d+m)\gamma n}{1-\gamma}+\frac{2m\log\rho_{\min}^{-1}}{\gamma C}.$$

The proof is completed by setting $\gamma=\frac{\sqrt{2m\log\rho_{\min}^{-1}}}{\sqrt{2m\log\rho_{\min}^{-1}}+\sqrt{C(Cm^{2}d+m)n}}.$

\section{Analysis of Algorithm \ref{alg:initucb} and \ref{alg:mixcombucb}}

\subsection{Proof of Theorem \ref{Thm:delta2}.}
Now we can establish a bound on the estimation error. We define the following martingales: $H_t(e)=R_t(e)-t\mu(e), \forall e\in\mathcal{A}, t\geq1$ with fillteration $\mathcal{F}_{n}=\sigma(w_1,w_2,...,w_n)$, we can verify that
$$\begin{aligned}
\mathbb{E}[H_t(e)|\mathcal{F}_{t-1}]&=H_{t-1}(e)+\mathbb{E}[w_{t}(e)\frac{\mathbb{I}\{e \in M(t)\}}{\mathbb{P}(e \in M(t))}|\mathcal{F}_{t-1}]-t\mu(e)  
=H_{t-1}(e)+\mathbb{E}[\frac{w_{t}(e)}{\mathbb{P}(e \in M(t))}|e \in M(t)]-t\mu(e)
\end{aligned}$$
and we can conclude that $\mathbb{E}[H_t(e)|\mathcal{F}_{t-1}]=H_{t-1}(e)$ and $(H_t(e))_{t\geq1}$ is martingale for all $e\in\mathcal{A}$.

As $\mathbb{P}(e \in M(t))\geq\mathbb{P}(M(t)=M_e)=\frac{1}{2m_0t^{\alpha}}\geq\frac{1}{2dt^{\alpha}}$, we can see that for all $t\geq1$ and $e\in\mathcal{A}$:

\begin{equation*}
|H_t(e)-H_{t-1}(e)|\leq |w_{t}(e)\frac{\mathbb{I}\{e \in M(t)\}}{\mathbb{P}(e \in M(t))}|+\mu(e)\leq |w_{t}(e)||\frac{1}{\mathbb{P}(e \in M(t))}|+\mu(e)\leq2dt^{\alpha}+1\leq3dt^{\alpha}.
\end{equation*}

For all $e\in\mathcal{A}$, we write $W_t(e)=\sum_{t=1}^{\tau}\mathbb{E}[(H_t(e)-H_{t-1}(e))^2|\mathcal{F}_{t-1}]$ and we have
\begin{equation*}
\begin{aligned}
\sum_{t=1}^{\tau}\mathbb{E}[(H_{t}(e)-H_{t-1}(e))^2|\mathcal{F}_{t-1}]&\leq \sum_{t=1}^{\tau}\mathbb{P}(e \in M(t))^{-2}\mathbb{E}[(w_{t}(e)\mathbb{I}\{e \in M(t)\})^2|\mathcal{F}_{t-1}]- \tau\mu(e)^2  
\\&\leq\sum_{t=1}^{\tau}\mathbb{P}(e \in M(t))^{-2}\mathbb{E}[\mathbb{I}\{e \in M(t)\})^2|\mathcal{F}_{t-1}]\leq\sum_{t=1}^{\tau}\mathbb{P}(e \in M(t))^{-1}
\end{aligned} 
\end{equation*}
where the second inequality holds due to $|w_t(e)|\leq1$. now we have $W_{\tau}(e)\leq\sum_{t=1}^{\tau}\mathbb{P}(e \in M(t))^{-1}\leq\sum_{t=1}^{\tau}2dt^{\alpha}$. We let $\hat{W}_{\tau}=\frac{2d(\tau+1)^{\alpha+1}}{\alpha+1}\geq W_{\tau}(e),\forall e\in\mathcal{A},\tau\geq1$ denote the determistic upper bound and we have for all $\delta\in(0,1]$:
$$\hat{W}_{t}\leq\frac{2d(\tau+1)^{\alpha+1}}{\alpha+1}\:\vee\frac{9d^2\tau^{\alpha+1}\ln(2/\delta)}{e-2}\leq\frac{9d^2\tau^{\alpha+1}\ln(2/\delta)}{e-2}.$$

by Bernstein’s Inequality, we know that with possibility more than $1- \delta$, we have
\begin{equation*}
|H_t(e)|\leq\sqrt{(e-2)\frac{9d^2t^{\alpha+1}(\ln(2/\delta))^2}{e-2}}=3d\ln(\frac{2}{\delta})t^{\frac{\alpha+1}{2}},\forall t\in[n].    
\end{equation*}

Applying union bound, we can se with probability more than $1-\delta$, there is $|H_t(e)|\leq3d\ln(\frac{2d}{\delta})t^{\frac{\alpha+1}{2}},\forall t\in[n],e\in\mathcal{A}$. Now we can easily derive that
\begin{equation*}
\begin{aligned}
|\hat{\Delta}_{t}^{(i,j)}-\Delta^{(i,j)}|&=|\frac{1}{t}(R_{t}(M(\tau_{i}))-f(M(\tau_{i}),\boldsymbol{\mu}))- \frac{1}{t}(R_{t}(M(\tau_{j})-f(M(\tau_{j}),\boldsymbol{\mu}))|  
\\&\leq\frac{1}{t}(\sum_{e_1\in  M(\tau_i)}|H_t(e_1)|+\sum_{e_2\in  M(\tau_j)}|H_t(e_2)|)\leq 6md\ln(\frac{2d}{\delta})\sqrt{t^{\alpha-1}}
\end{aligned}  
\end{equation*}

take $\delta=\frac{1}{n^2}$ and we finished our proof.

\subsection{Proof of Theorem \ref{Thm:R2}.}
Now we introduce our first lemma used in our proof.
\begin{lemma}
\label{Le:Ft}
Let:
\begin{equation}
\label{Eq:Ft}
\mathcal{F}_{t}=\left\{\Delta_{  M(t)}\leq2\sum_{e\in\hat{M}(t)}c_{n,T_{t-1}(e)},\:\Delta_{  M(t)}>0\right\}   
\end{equation}    
be the event that suboptimal solution $M(t)$ is "hard to distinguish" from $M^*$ at time $t$, where $\hat{M}(t) =   M(t) \setminus   M^*$. Let $\mathcal{K}_t=\{M(t)=\tilde M(t)\}$ to denote the event that UCB best arm is chosen by algorithm. Then the regret of \texttt{MixCombUCB} is bounded as:
\begin{equation*}
\mathcal{R}_{\nu}(n,\pi)\leq\mathbb{E}[\hat{\mathcal{R}}_{\nu}(n,\pi)]+5md+\frac{mn^{1-\alpha}}{2(1-\alpha)} 
\end{equation*}
where:
\begin{equation}
\label{Eq:Rv(npi)}
\hat{\mathcal{R}}_{\nu}(n,\pi)=\sum_{t=m_0}^{n}\Delta_{  M(t)}\mathbb{I}\{\mathcal{F}_{t}\}\mathbb{I}\{\mathcal{K}_t\}\:.  \end{equation}
\end{lemma}

The key step in our analysis is that we define a cascade of infinitely-many mutually-exclusive events and then bound the number of times that these events happen when a sub-optimal solution is chosen. The events are parametrized by two decreasing sequences of constants:
\begin{equation}
\label{Eq:betaseq}
1=\beta_0>\beta_1>\beta_2>\ldots>\beta_k>\ldots    
\end{equation}
\begin{equation}
\label{Eq:alphaseq}
\alpha_1 > \alpha_2 > \ldots > \alpha_k > \ldots   
\end{equation}
such that $\lim_{i\to\infty}\alpha_i=\lim_{i\to\infty}\beta_i=0$. We define:
$$m_{i,t}=\frac{\alpha_im^2}{\Delta_{  M(t)}^2}\log n$$
and assume that $m_{i,t}=\infty$ when $\Delta_{  M(t)}=0$. The events at time $t$ are defined as in \cite{pmlr-v38-kveton15}:
$$\begin{aligned}  
G_{1,t} = \{\text{at least } \beta_1 m \text{ items in } \hat{M}(t) \text{ were observed at most } m_{1,t} \text{ times}\}
\end{aligned}$$
$$\begin{aligned}  
G_{2,t} = \{&\text{less than } \beta_1 m \text{ items in } \hat{M}(t) \text{ were observed at most } m_{1,t} \text{ times},
\\&\text{at least } \beta_2 m \text{ items in } \hat{M}(t) \text{ were observed at most } m_{2,t} \text{ times}\}
\end{aligned}$$
$$\begin{aligned}
...    
\end{aligned}$$
$$\begin{aligned}  
G_{i,t} = \{&\text{less than } \beta_1 m \text{ items in } \hat{M}(t) \text{ were observed at most } m_{1,t} \text{ times},
\\&\dots,
\\&\text{less than } \beta_{i-1} m \text{ items in } \hat{M}(t) \text{ were observed at most } m_{i-1,t} \text{ times},
\\&\text{at least } \beta_i m \text{ items in } \hat{M}(t) \text{ were observed at most } m_{i,t} \text{ times}\}
\end{aligned}$$
$$\begin{aligned}
...    
\end{aligned}$$

The following lemma establishes a sufficient condition under which events $G_{i,t}$ are exhaustive. This is the key step in the proofs in this section. We firstly introduce an lemma from Lemma 3 of \cite{pmlr-v38-kveton15}:

\begin{lemma}
\label{Le:alphabetaineq}
Let $(\alpha_i)$ and $(\beta_i)$ be defined as in \ref{Eq:betaseq} and \ref{Eq:alphaseq}, respectively; and let:
\begin{equation}
\label{Eq:alphabetaineq}
2\sqrt{2} \sum_{i=1}^{\infty} \frac{\beta_{i-1}-\beta_{i}}{\sqrt{\alpha_{i}}} \leq 1 .    
\end{equation}
Let event $\mathcal{F}_t$ happen. Then event $G_{i,t}$ happens for some $i$.    
\end{lemma}

Now we continue our proof. Let $\mathcal{F}_{t}$ be the event in \ref{Eq:Ft}. By Lemma \ref{Le:Ft} and \ref{Le:alphabetaineq}, it remains to bound

$$\hat{\mathcal{R}}_{\nu}(n,\pi)=\sum_{t=m_0}^n\Delta_{  M(t)}\mathbb{I}\{\mathcal{F}_t\}\mathbb{I}\{\mathcal{K}_t\}=\sum_{i=1}^\infty\sum_{t=m_0}^n\Delta_{  M(t)}\mathbb{I}\{G_{i,t},\Delta_{  M(t)}>0\}\mathbb{I}\{\mathcal{K}_t\}\:.$$

In the next step, we define item-specific variants of events $G_{i,t}$ and associate the regret at time $t$ with these events. In
particular, let:

$$G_{e,i,t}=G_{i,t}\cap\left\{e\in\hat{M}(t),T_{t-1}(e)\leq m_{i,t}\right\}$$

be the event that item $e$ is not observed “sufficiently often” under event $G_{i,t}.$ Then it follows that:
$$\mathbb{I}\{G_{i,t},\Delta_{  M(t)}>0\}\leq\frac{1}{\beta_im}\sum_{e\in\tilde{\mathcal{A}}}\mathbb{I}\{G_{e,i,t},\Delta_{  M(t)}>0\}\:,$$

because at least $\beta_im$ items are not observed “sufficiently often” under event $G_{i,t}.$ Therefore, we can bound $\hat{\mathcal{R}}_{\nu}(n,\pi)$ as:

$$\hat{\mathcal{R}}_{\nu}(n,\pi)\leq\sum\limits_{e\in\tilde{\mathcal{A}}}\sum\limits_{i=1}^{\infty}\sum\limits_{t=m_0}^{n}\mathbb{I}\{G_{e,i,t},\Delta_{  M(t)}>0\}\:\frac{\Delta_{  M(t)}}{\beta_im}\:.$$

Let each item $e$ be contained in $N_e$ suboptimal super arms and $\Delta_{e,1}\geq\ldots\geq\Delta_{e,N_e}$ be the gaps of these super arms, ordered from the largest gap to the smallest one, we denote those super arms as $M_{e,1},\ldots,M_{e,N_e}$. Then $\hat{\mathcal{R}}_{\nu}(n,\pi)$ can be further bounded as:

$$\begin{aligned}\hat{\mathcal{R}}_{\nu}(n,\pi)&\leq\sum_{e\in\tilde{\mathcal{A}}}\sum_{i=1}^{\infty}\sum_{t=m_0}^{n}\sum_{k=1}^{N_{e}}\mathbb{I}\left\{G_{e,i,t},M(t)=M_{e,k}\right\}\mathbb{I}\{\mathcal{K}_t\}\frac{\Delta_{e,k}}{\beta_{i}m}
\\&\stackrel{(a)}{\leq}\sum_{e\in\tilde{\mathcal{A}}}\sum_{i=1}^{\infty}\sum_{t=m_0}^{n}\sum_{k=1}^{N_{e}}\mathbb{I}\left\{e\in\hat{M}(t),T_{i-1}(e)\leq\alpha_{i}\frac{m^{2}}{\Delta_{e,k}^{2}}\log n\right\}\mathbb{I}\left\{M(t)=M_{e,k},\mathcal{K}_t\right\}\frac{\Delta_{e,k}}{\beta_{i}m}
\\&\stackrel{(b)}{\leq}\sum_{e\in\tilde{\mathcal{A}}}\sum_{i=1}^{\infty}\frac{\alpha_{i}m\log n}{\beta_{i}}\left[\Delta_{e,1}\frac{1}{\Delta_{e,1}^{2}}+\sum_{k=2}^{N_{e}}\Delta_{e,k}\left(\frac{1}{\Delta_{e,k}^{2}}-\frac{1}{\Delta_{e,k-1}^{2}}\right)\right]\end{aligned}$$

where inequality (a) is by the definition of event $G_{e,i,t}$, inequality (b) follows from the solution to:

\begin{equation}
\label{Eq:solvemax}
\max_{M(m_0),\ldots,M(n)}\sum_{t=m_0}^{n}\sum_{k=1}^{N_{e}}\mathbb{I}\left\{e\in\hat{M}(t),T_{i-1}(e)\leq\alpha_{i}\frac{m^{2}}{\Delta_{e,k}^{2}}\log n\right\}\mathbb{I}\left\{M(t)=M_{e,k},\mathcal{K}_t\right\}\frac{\Delta_{e,k}}{\beta_{i}m}\:,    
\end{equation}
The solution of \ref{Eq:solvemax} is based on three observations: First, the gaps have order $\Delta_{e,1}\geq\ldots\geq\Delta_{e,N_e}$.
Second, by the design of our algorithm, the counter $T_e(t)$ increases when the event $\mathbb{I}\left\{M(t)=M_{e,k},\mathcal{K}_t\right\}$ happens. 
Finally, it's easy to see that $\sum_{k=1}^{N_e}\mathbb{I}\left\{M(t)=M_{e,k},\mathcal{K}_t\right\}\leq1$ for any given $e$ and $t$. Based on this facts, we can preceed with inequality (b). From Lemma 3 of \cite{10.5555/3020751.3020795}, we have that:
\begin{equation}
\label{Eq:boundaibi}
\hat{\mathcal{R}}_{\nu}(n,\pi)\leq\sum_{e\in\tilde{\mathcal{A}}}\sum_{i=1}^{\infty}\frac{\alpha_{i}m\log n}{\beta_{i}}\frac{2}{\Delta_{e,\operatorname*{min}}}=\sum_{e\in\tilde{\mathcal{A}}}\frac{2m}{\Delta_{e,\min}}\left[\sum_{i=1}^{\infty}\frac{\alpha_{i}}{\beta_{i}}\right]\log n    
\end{equation}

It remain to choose $(\alpha_i)$ and $(\beta_i)$ such that:

\begin{itemize}
    \item $\lim_{i \to \infty} \alpha_i = \lim_{i \to \infty} \beta_i = 0.$

    \item Monotonicity conditions in Equations~\ref{Eq:alphaseq} and~\ref{Eq:betaseq} hold.

    \item Inequality~\ref{Eq:alphabetaineq} holds: $2\sqrt{2} \sum_{i=1}^{\infty} \frac{\beta_{i-1} - \beta_i}{\sqrt{\alpha_i}} \leq 1.$

    \item $\sum_{i=1}^{\infty} \frac{\alpha_i}{\beta_i}$ is minimized.
\end{itemize}

Similar to the proof of Lemma 3 in \cite{pmlr-v38-kveton15}, we choose $(\alpha_i)$ and $(\beta_i)$ to be geometric sequences, 
$\beta_i = b^i$ and $\alpha_i = l a^i$ for $0 < a, b < 1$ and $l > 0$. 
For this setting, $\alpha_i \to 0$ and $\beta_i \to 0$, and the monotonicity conditions are also satisfied. 
Moreover, if $b \leq \sqrt{a}$, we have:
$$2\sqrt{2} \sum\limits_{i=1}^{\infty} \frac{\beta_{i-1} - \beta_i}{\sqrt{\alpha_i}} 
= 2\sqrt{2} \sum\limits_{i=1}^{\infty} \frac{b^{i-1} - b^i}{\sqrt{l a^i}} 
= 2\sqrt{\frac{2}{l}} \cdot \frac{1 - b}{\sqrt{a} - b} \leq 1$$

provided that $l \geq 8 \left( \frac{1 - b}{\sqrt{a} - b} \right)^2$.
Furthermore, if $a < b$, we have:
$$\sum_{i=1}^\infty \frac{\alpha_i}{\beta_i} 
= \sum_{i=1}^\infty \frac{l a^i}{b^i} 
= \frac{l a}{b - a}.$$
Given the above, the best choice of $l$ is $8 \left( \frac{1 - b}{\sqrt{a} - b} \right)^2$, and the problem of minimizing the constant in our regret bound can be written as:
$$\begin{aligned}
\inf_{a, b} \quad & 
8 \left( \frac{1 - b}{\sqrt{a} - b} \right)^2 \cdot \frac{a}{b - a} \\
\text{s.t.} \quad & 
0 < a < b < \sqrt{a} < 1.
\end{aligned}$$

We find the solution to the above problem numerically, and determine it to be $a = 0.1459$ and $b = 0.2360$. 
For these values of $a$ and $b$, $8 \left( \frac{1 - b}{\sqrt{a} - b} \right)^2 \cdot \frac{a}{b - a} < 356.$
We apply this upper bound to Equation \ref{Eq:Rv(npi)} and \ref{Eq:boundaibi}, and it follows that the regret is bounded as:
\begin{equation*}
\mathcal{R}_{\nu}(n,\pi)\leq\mathbb{E}[\hat{\mathcal{R}}_{\nu}(n,\pi)]+5md+\frac{mn^{1-\alpha}}{2(1-\alpha)}\leq \sum_{e\in\tilde{\mathcal{A}}}\frac{712m}{\Delta_{e,\min}}\log n+5md+\frac{mn^{1-\alpha}}{2(1-\alpha)}    
\end{equation*}

\subsection{Proof of Proposition \ref{Thm:R3}.}

The key idea is to decompose the regret into two parts, where the gaps are larger than $\epsilon$ and at most $\epsilon.$ We analyze each part separately and then set $\epsilon$ to get the desired result.

By \ref{Le:Ft} ,it remains to bound $\hat{\mathcal{R}}_{\nu}(n,\pi)=\sum_{t=m_0}^n\Delta_{  M(t)}\mathbb{I}\{\mathcal{F}_t\}\mathbb{I}\{\mathcal{K}_t\}$, where the event $\mathcal{F}_t$ is defined in \ref{Eq:Ft}. We partition $\hat{\mathcal{R}}_{\nu}(n,\pi)$ as:

$$\begin{aligned}\hat{\mathcal{R}}_{\nu}(n,\pi)&=\sum_{t=m_0}^n\Delta_{  M(t)}\mathbb{I}\{\mathcal{F}_t,\Delta_{  M(t)}<\epsilon\}\mathbb{I}\{\mathcal{K}_t\}+\sum_{t=m_0}^n\Delta_{  M(t)}\mathbb{I}\{\mathcal{F}_t,\Delta_{  M(t)}\geq\epsilon\}\mathbb{I}\{\mathcal{K}_t\}\\&\leq\epsilon n+\sum_{t=m_0}^n\Delta_{  M(t)}\mathbb{I}\{\mathcal{F}_t,\Delta_{  M(t)}\geq\epsilon\}\mathbb{I}\{\mathcal{K}_t\}\end{aligned}$$

and bound the first term trivially. The second term is bounded in the same way as $\hat{\mathcal{R}}_{\nu}(n,\pi)$ in the proof of Theorem \ref{Thm:R2}, 
except that we only consider the gaps $\Delta_{e,k}\geq\epsilon.$ Therefore, $\Delta_{e,\min}\geq\epsilon$ and we get:

$$\begin{aligned}\sum_{t=m_0}^n\Delta_{  M(t)}\mathbb{I}\{\mathcal{F}_t,\Delta_{  M(t)}\geq\epsilon\}\mathbb{I}\{\mathcal{K}_t\}\leq\sum_{e\in\tilde{\mathcal{A}}}\frac{712m}{\epsilon}\log n\leq \frac{712}{\epsilon}md\log n\:.\end{aligned}$$

Based on the above inequalities:

$$\mathcal{R}_{\nu}(n,\pi)\leq\dfrac{712md}{\epsilon}\log n+\epsilon n+5md+\frac{mn^{1-\alpha}}{2(1-\alpha)}\:.$$

Finally, we choose $\epsilon=\sqrt{\frac{712md\log n}{n}}$ and get:

$$\mathcal{R}_{\nu}(n,\pi)\leq2\sqrt{712mdn\log n}+5md+\frac{mn^{1-\alpha}}{2(1-\alpha)}<54\sqrt{mdn\log n}+5md+\frac{mn^{1-\alpha}}{2(1-\alpha)}\:,$$

which concludes our proof.    

\section{Analysis of Pareto Optimality}

\subsection{Proof of Theorem \ref{Thm:ParetoM}}
First, let us consider the Pareto optimality defined in classic multi-armed bandit with $K$ arms introduced in \cite{doi:10.1287/mnsc.2023.00492} which proposes the necessary and sufficient condition for Pareto optimality in MAB with general $K$ arms and time horizon $T$. Specifically, an admissible pair $(\pi^*,\widehat{\Delta}^*)$ is Pareto optimal if and only if

$$\max\limits_{\nu\in\mathcal{V}_0}\left[\left(\max\limits_{i<j\le K}\mathcal{E}\big(T,\Delta^{*(i,j)}\big)\right)\sqrt{\mathcal{R}_{\nu}\big(T,\pi^{*}\big)}\right]=\tilde{\mathcal{O}}(1).$$

In the Combinatorial Bandit setting, we can see each super arm as an arm, thus we have $|\mathcal{M}|$ arms and each arm has its own reward distribution. Speciafically, arm $M_{\tau}$ has a reward distribution with mean $\mu_{\tau}=\sum_{e\in M_{\tau}}\mu(e)$. Therefore, it follows that there also exists Pareto optimality in the context of Combinatorial Bandit and the necessary and sufficient condition is

$$\max_{\nu\in\mathcal{V}_0}\left[\left(\max_{i<j\leq|\mathcal{M}|}\mathcal{E}\big(n,\Delta_{M}^{(i,j)}\big)\right)\sqrt{\mathcal{R}_{\nu}\big(n,\pi\big)}\right]=\bar{\mathcal{O}}(1).$$

where $\mathcal{E}(n,\Delta_{M}^{(i,j)})$ is the estimation error of ATE between $M_{\tau_i}$ and $M_{\tau_j}$ with $\mathcal{R}_{\nu}\big(n,\pi\big)$ being the cumulative regret within $n$ time steps under policy $\pi.$

\subsection{Proof of Lemma \ref{Thm:Paretoinf}}
We first come up with a lemma:
\begin{lemma}
\label{Le:fineq}
When $|\mathcal{A}_{ad}|=2$, for any given online decision-making policy $\pi,$ the error of any estimator of parameter difference can be lower bounded as follows, for any function $f:n\to[0,\frac{1}{8}]$ and any $u\in\mathcal{E}.$
$$\inf_{\widehat\Delta_{\mu}}\operatorname*{max}_{\nu\in\mathcal{V}_0}\mathbb P_\nu\left(|\widehat\Delta_{\mu}-\Delta_\nu|\geq g(t)\right)\geq\frac{1}{2}\left[1-\sqrt{\frac{16g(t)^2\mathcal{R}_{u}\big(n,\pi\big)}{3|\Delta_u|}}\right].$$  
\end{lemma}

Now we can proceed on our proof. Based on Lemma \ref{Le:fineq}, given policy $\pi$, and $\hat{\Delta}_n$, if $g(n)\leq\sqrt\frac{3|\Delta_u|}{64\mathcal{R}_{u}\big(n,\pi\big)}$ for some $u\in\mathcal{V}_0$, there is
$$\begin{aligned}\max_{\nu\in\mathcal{V}_0}\mathbb{E}\left[|\widehat{\Delta}_{\mu}-\Delta_{\nu}|\right]&\geq g(n)\max_{\nu\in\mathcal{V}_0}\mathbb{P}_{\nu}\left(|\widehat{\Delta}_{\mu}-\Delta_{\nu}|_{2}\geq g(n)\right)\\&\geq\frac{g(n)}{2}\left[1-\sqrt{\frac{16g(n)^{2}}{3\Delta_{u}}\mathcal{R}_{u}\big(n,\pi\big)}\right]\geq\frac{g(n)}{4},\end{aligned}$$

where the second inequality holds because of Lemma \ref{Le:fineq}. We use $\nu_{\pi,\hat{\Delta}_{\mu}}$ to denote $\arg\max_{\nu\in\mathcal{V}_0}\mathbb{E}\left[|\widehat{\Delta}_\mu-\Delta_\nu|\right]$ 
given policy $\pi$ and $\widehat\Delta_\mu$, and thus we have $e_{\nu_{\pi,\hat{\Delta}_\mu}}(n,\widehat{\Delta}_\mu)\geq\frac{g(n)}{4}$.
After choosing $g(n)=\sqrt{\frac{3|\Delta_{\nu_{\pi,\hat{\Delta}_{\mu}}}|}{64\mathcal{R}_{\nu_{\pi,\hat{\Delta}_{v}}}(n,\pi)}}$,we retrieve for any given policy $\pi$ and $\widehat\Delta_v$,

$$\begin{aligned}\max_{\nu\in\mathcal{V}_0}\left[e_{\nu}(n,\widehat{\Delta}_{\mu})\sqrt{\mathcal{R}_{\nu}(n,\pi)}\right]&\geq e_{\nu_{\pi,\bar{\Delta}_{\mu}}}(n,\widehat{\Delta}_{\mu})\sqrt{\mathcal{R}_{\nu_{\pi,\bar{\Delta}_{\mu}}}(n,\pi)}\\&\geq\frac{g(n)}{4}\sqrt{\mathcal{R}_{\nu_{\pi,\bar{\Delta}_{\mu}}}(n,\pi)}=\Theta(1),\end{aligned}$$

where the last equation holds because we consider respective $g(n)$ and $\Delta_\nu=\Theta(1)$ for $\nu\in\mathcal{V}_0.$ Since the above
inequalities hold true for any policy $\pi$ and $\hat{\Delta}_\mu$, we finish the proof.

\subsection{Proof of Lemma \ref{Thm:Paretosup}}

We proceed by contradiction. Suppose that $(\pi_0, \widehat{\Delta}_0)$ satisfies the above condition, but is not Pareto optimal. Then, there exists a pair $(\pi_1, \widehat{\Delta}_1)$ that Pareto dominates $(\pi_0, \widehat{\Delta}_0)$. According to the lower bound established in Theorem \ref{Thm:Paretoinf}, there must exist a point on the Pareto front of $(\pi_1, \widehat{\Delta}_1)$, denoted as 

$$(\mathcal{E}_{\nu_1}(n, \widehat{\Delta}_1), \mathcal{R}_{\nu_1}(n, \pi_1)),$$ 

such that 

$$\mathcal{E}_{\nu_1}(n, \widehat{\Delta}_1) \sqrt{\mathcal{R}_{\nu_1}(n, \pi_1) } = \Omega(1).$$

By the definition of Pareto dominance, there must exist a point 

$$(\mathcal{E}_{\nu_2}(n, \widehat{\Delta}_0), \mathcal{R}_{\nu_2}(n, \pi_0)) \in \mathcal{F}(\pi_0, \widehat{\Delta}_0)$$ 

such that
$$\mathcal{E}_{\nu_2}(n, \widehat{\Delta}_0)\sqrt{\mathcal{R}_{\nu_2}(n, \pi_0) } > \mathcal{E}_{\nu_1}(n, \widehat{\Delta}_1)\sqrt{ \mathcal{R}_{\nu_1}(n, \pi_1) } = \Omega(1).$$

Note that the inequality above is strict with respect to the dependence on $n$. This implies that 
$$\mathcal{E}_{\nu_2}(n, \widehat{\Delta}_0) \cdot \sqrt{ \mathcal{R}_{\nu_2}(n, \pi_0) } = \Omega(n^p)$$ 
for some constant $p>0$, which contradicts our original assumption.

\section{Proof of Technical Lemmas}

\subsection{Proof of Lemma \ref{Le:qtKL}}
We have
$$\mathrm{KL}(q,\tilde{q}_{t})-\mathrm{KL}(q,q_{t-1})=\sum_{i\in[d]}q(i)\log\frac{q_{t-1}(i)}{\tilde{q}_{t}(i)}=-\eta\sum_{i\in[d]}q(i)\tilde{w}_t(i)+\log Z_{t},$$   
with
\begin{equation}
\label{Eq:logZt}
\begin{aligned}
\log Z_t&=\log\sum_{i\in[d]}q_{t-1}(i)\exp\left(\eta\tilde{w}_t(i)\right)
\\&\leq\log\sum_{i\in[d]}q_{t-1}(i)\left(1+\eta\tilde{w}_t(i)+\eta^{2}\tilde{w}^{2}_t(i)\right)\\&\leq\eta q_{t-1}^{\top}\tilde{w}_t+\eta^{2}q_{t-1}^{\top}\tilde{w}^{2}_t,
\end{aligned}    
\end{equation}
where we used $\exp(z) \leq 1 + z + z^2$ for all $|z| \leq 1$ in the first inequality and $\log(1 + z) \leq z$ for all $z > -1$ in the second inequality. Later we verify the condition for the former inequality.

Hence we have
$$\operatorname{KL}(q,\tilde{q}_{t})-\operatorname{KL}(q,q_{t-1})\leq\eta q_{t-1}^{\top}\tilde{w}_t-\eta q^{\top}\tilde{w}_t+\eta^{2}q_{t-1}^{\top}\tilde{w}^{2}_t.$$
Generalized Pythagorean inequality (see Theorem 3.1 in \cite{Csiszar2004}) basically says that if $U_t=0$ and 
$q_t=\arg\min_{p\in\mathcal{Q}}\sum_{i\in[d]}p(i)\log\frac{p(i)}{\tilde{q}_t(i)}$, there is
$$\mathrm{KL}(q,q_t)+\mathrm{KL}(q_t,\tilde{q}_t)\leq\mathrm{KL}(q,\tilde{q}_t).$$
Since KL$(q_{t},\tilde{q}_{t})\geq0$, we get if $U_t=0$,
$$\mathrm{KL}(q,q_{t})-\mathrm{KL}(q,q_{t-1})\le\eta q_{t-1}^{\top}\overline{w}_t-\eta q^{\top}\overline{w}_t+\eta^2 q_{t-1}^{\top}\overline{w}^2_t.$$
Obviously if $U_t=1$, $q_t=q_{t-1}$, therefore, summing over $t$ gives
$$\begin{aligned}\sum_{t\in V_n}\left(q^\top\tilde{w}_t-q_{t-1}^\top\tilde{w}_t\right)\leq\eta\sum_{t\in V_n}q_{t-1}^\top\tilde{w}^2_t+\frac{\mathrm{KL}(q,q_0)}{\eta}.\end{aligned}$$
To satisfy the inequality condition in \ref{Eq:logZt}, i.e., $\eta|\tilde{w}_t(i)|\leq1,\:\forall i\in[d],$ we find the upper bound for
$\max_{i\in[d]}|\tilde{w}_t(i)|$ as follows:
$$\begin{aligned}\max_{i\in[d]}|\tilde{w}_t(i)|&\leq\|\tilde{w}_t\|_{2}=\|\Sigma_{t-1}^{+}\boldsymbol{\theta}_{M(t)}f(M(t),w_t)\|_{2}\leq m\|\Sigma_{t-1}^{+}\boldsymbol{\theta}_{M(t)}\|_{2}\\&\leq m\sqrt{\boldsymbol{\theta}_{M(t)}^{\top}\Sigma_{t-1}^{+}\Sigma_{t-1}^{+}\boldsymbol{\theta}_{M(t)}}\leq m\|\boldsymbol{\theta}_{M(t)}\|_{2}\sqrt{\lambda_{\max}\left(\Sigma_{t-1}^{+}\Sigma_{t-1}^{+}\right)}\\&=m^{3/2}\sqrt{\lambda_{\max}\left(\Sigma_{t-1}^{+}\Sigma_{t-1}^{+}\right)}=m^{3/2}\:\lambda_{\max}\left(\Sigma_{t-1}^{+}\right)=\frac{m^{3/2}}{\lambda_{\min}\left(\Sigma_{t-1}\right)},\end{aligned}$$
where $\lambda_{\mathrm{max}}(A)$ and $\lambda_{\mathrm{min}}(A)$ respectively denote the maximum and the minimum nonzero eigenvalue of matrix $A.$ Note that $\rho^{0}$ induces uniform distribution over $\mathcal{M} .$ Thus by $q_{t-1}^{\prime}=$ $(1-\gamma)q_{t-1}+\gamma\rho^0$ we see that $p_{t-1}$ is a mixture of uniform distribution and the distribution induced by $q_{t-1}.$ Note that, we have:
$$\begin{array}{r} x^{\top}\Sigma_{t-1}x = \mathbb{E}\left[x^{\top}\boldsymbol{\theta}_{M(t)}\boldsymbol{\theta}_{M(t)}^{\top}x\right] = \mathbb{E}\left[\left(\boldsymbol{\theta}_{M(t)}^{\top}x\right)^2\right] \geq \gamma\mathbb{E}\left[\left(\boldsymbol{\theta}_{M}^{\top}x\right)^2\right], \end{array}$$
where in the last inequality $M$ has law $\rho^{0}$. By definition, we have for any $x \in \mathrm{span}(\boldsymbol{\theta})$ with $\|x\|_{2} = 1$, there is $\mathbb{E}\left[(\boldsymbol{\theta}_{M}^{\top}x)^{2}\right]\geq\lambda_{\min}$, so that in the end, we get $\lambda_{\min}(\Sigma_{n-1})\geq\gamma\lambda_{\min}$, and hence $\eta|\tilde{w}_t(i)|\leq\frac{\eta m^3/2}{\gamma\lambda_{\min}},\forall i\in[d].$ Finally, we choose $\eta\leq\frac{\gamma\lambda_{\min}}{m^{3/2}}$ to satisfy the condition for the inequality we used in \ref{Eq:logZt}.

\subsection{Proof of Lemma \ref{Le:Ft}.}
Let $R_t = R(M(t), w_t)$ be the stochastic regret of \texttt{MixCombUCB} at time $t$, where $M(t)$ and $w_t$ are the super arm and the weights of the items at time $t$, respectively. Furthermore, let $\mathcal{E}_t = \{\exists e \in \mathcal{A}: \left| \mu(e) - \hat{w}_{T_{t-1}(e)}(e) \right|$$ \geq c_{t-1, T_{t-1}(e)} \}$ be the event that $\mu(e)$ is outside of the high-probability confidence interval around $\hat{w}_{T_{t-1}(e)}(e)$ for some item $e$ at time $t$; $\mathcal{K}_t=\{M(t)=\tilde M(t)\}$ to denote the event that UCB best arm is chosen by algorithm; and let $\bar{\mathcal{E}}_t$ be the complement of $\mathcal{E}_t$, $\mu(e)$ is in the high-probability confidence interval around $\hat{w}_{T_{t-1}(e)}(e)$ for all $e$ at time $t$. Then we can decompose the regret of CombUCB1 as:
$${\mathcal {R}}_{\nu}(n,\pi) = \mathbb{E}\left[\sum_{t=1}^{m_0-1} R_t\right] + \mathbb{E}\left[\sum_{t=m_0}^{n} \mathbb{I}\left\{\mathcal{E}_t\right\}\mathbb{I}\{\mathcal{K}_t\} R_t\right] + \mathbb{E}\left[\sum_{t=m_0}^{n} \mathbb{I}\left\{\bar{\mathcal{E}}_t\right\}\mathbb{I}\left\{\mathcal{K}_t\right\} R_t\right]+\mathbb{E}\left[\sum_{t=m_0}^{n}\mathbb{I}\{\bar{\mathcal{K}}_t\}R_t\right].$$

Now we bound each term in our regret decomposition.

The regret of the initialization, $\mathbb{E}\left[\sum_{t=1}^{m_0-1} R_t\right]$, is bounded by $md$ because Algorithm 2 terminates in at most $d$ steps, and $R_t \leq m$ for any $M(t)$ and $w_t$.
The second term in our regret decomposition, $\mathbb{E}\left[\sum_{t=m_0}^n \mathbb{I}\{\mathcal{E}_t\}\mathbb{I}\{\mathcal{K}_t\} R_t\right]$, is small because all of our confidence intervals hold with high probability. In particular, for any $e$, $s$, and $t$:
$$P\left(|\mu(e)-\hat{w}_s(e)| \geq c_{t,s}\right) \leq 2 \exp[-3 \log t] \,,$$
and therefore:
$$
\mathbb{E}\left[\sum_{t=m_0}^n \mathbb{I}\{\mathcal{E}_t\}\mathbb{I}\{\mathcal{K}_t\}\right] \leq \sum_{e \in E} \sum_{t=1}^n \sum_{s=1}^t P\left(|\mu(e)-\hat{w}_s(e)| \geq c_{t,s}\right) \leq 2 \sum_{e \in E} \sum_{t=1}^n \sum_{s=1}^t \exp[-3 \log t] \leq 2 \sum_{e \in E} \sum_{t=1}^n t^{-2} \,.$$

Since $R_t \leq m$ for any $M(t)$ and $w_t$, $\mathbb{E}\left[\sum_{t=m_0}^n \mathbb{I}\left\{\mathcal{E}_t\right\}\mathbb{I}\{\mathcal{K}_t\} R_t\right] \leq 4md$.

Finally, we rewrite the last term in our regret decomposition as:

$$\mathbb{E}\left[\sum_{t=m_0}^n \mathbb{I}\left\{\overline{\mathcal{E}}_t\right\}\mathbb{I}\{\mathcal{K}_t\} R_t\right] \stackrel{(a)}{=} \sum_{t=m_0}^n \mathbb{E}\left[\mathbb{I}\left\{\overline{\mathcal{E}}_t\right\}\mathbb{I}\{\mathcal{K}_t\} \mathbb{E}\left[R_t \mid M(t)\right]\right] \stackrel{(b)}{=} \mathbb{E}\left[\sum_{t=m_0}^n \Delta_{  M(t)} \mathbb{I}\left\{\overline{\mathcal{E}}_t, \Delta_{  M(t)} > 0\right\}\mathbb{I}\{\mathcal{K}_t\}\right].$$

In equality (a), the outer expectation is over the history of the agent up to time $t$,which in turn determines $M(t)$ and $\overline{\mathcal{E}}_t;$ and $\mathbb{E}\left[R_t\mid M(t)\right]$is the expected regret at time $t$ conditioned on super arm $M(t).$ Equality (b) follows from $\Delta_{  M(t)}=\mathbb{E}\left[R_t\mid M(t)\right].$ Now we bound $\Delta_{  M(t)}1\{\overline{\mathcal{E}}_t,\Delta_{  M(t)}>0\}$ for any suboptimal $M(t).$ The bound is derived based on two facts. First, when \texttt{MixCombUCB} chooses $  M(t),f(M(t),U_t)\geq f(M^*,U_t).$ This further implies that $\sum_{e\in   M(t)\setminus   M^*}U_t(e)\geq\sum_{e\in   M^*\setminus   M(t)}U_t(e).$ Second, when event $\overline{\mathcal{E}}_{t}$ happens, $\left|\mu(e)-\hat{w}_{T_{t-1}(e)}(e)\right|<c_{t-1,T_{t-1}(e)}$ for all items $e.$ Therefore:

$$\sum_{e\in   M(t)\setminus   M^*}\mu(e)+2\sum_{e\in  M(t)\setminus   M^*}c_{t-1,T_{t-1}(e)}\geq\sum_{e\in   M(t)\setminus   M^*}U_{t}(e)\geq\sum_{e\in   M^*\setminus   M(t)}U_{t}(e)\geq\sum_{e\in   M^*\setminus   M(t)}\mu(e)\:,$$

and $2\sum_{e\in M(t)\setminus   M^*}c_{t-1,T_{t-1}}(e)\geq\Delta_{  M(t)}$ follows from the observation that $\Delta_{  M(t)}=\sum_{e\in   M^*\setminus   M(t)}\mu(e)$ $-\sum_{e\in M(t)\setminus   M^*}\mu(e)$. Now note that $c_{n,T_{t-1}}(e)\geq c_{t-1,T_{t-1}(e)}$ for any time $t\leq n$. Therefore, the event $\mathcal{F}_t$ must happen and:
$$\mathbb{E}\left[\sum_{t=m_0}^n\Delta_{  M(t)}\mathbb{I}\left\{\overline{\mathcal{E}}_t,\Delta_{  M(t)}>0\right\}\mathbb{I}\{\mathcal{K}_t\}\right]\leq\mathbb{E}\left[\sum_{t=m_0}^n\Delta_{  M(t)}\mathbb{I}\left\{\mathcal{F}_t\right\}\mathbb{I}\{\mathcal{K}_t\}\right].$$

Noticed that
$$\mathbb{E}\left[\sum_{t=m_0}^{n}\mathbb{I}\{\bar{\mathcal{K}}_{t}\}R_{t}\right]\leq m\mathbb{E}\left[\sum_{t=m_0}^{n}\mathbb{I}\{\bar{\mathcal{K}}_{t}\}\right]=\frac{m}{2}\sum_{t=m_0}^{n}\frac{1}{t^{\alpha}}\leq\frac{mn^{1-\alpha}}{2(1-\alpha)}$$

This conclude our proof.

\subsection{Proof of Lemma \ref{Le:alphabetaineq}.}
We follow the proof framework of \cite{pmlr-v38-kveton15} and we firstly fix $t$ such that $\Delta_{  M(t)}>0.$ Because $t$ is fixed, we
use shorthands $G_i=G_{i,t}$ and $m_i=m_{i,t}.$ Letr

$$S_i=\left\{e\in\hat{M}(t):T_{t-1}(e)\leq m_i\right\}$$

be the set of items in $\hat{M}(t)$ that are not observed “sufficiently
often" under event $G_i$. Then event $G_i$ can be written as:

$$G_i=\left(\bigcap_{j=1}^{i-1}\left\{|S_j|<\beta_jm\right\}\right)\cap\{|S_i|\ge\beta_im\}\:.$$

Now we prove that event $G_i$ happens for some $i$ by showing that the event that none of our events happen cannot happen. Note that this event can be written as:

$$\begin{aligned}\bar{G}=\overline{\bigcup_{i=1}^{\infty}G_{i}}=\bigcap_{i=1}^{\infty}\left[\left(\bigcup_{j=1}^{i-1}\{|S_{j}|\geq\beta_{j}m\}\right)\cup\{|S_{i}|<\beta_{i}m\}\right]=\bigcap_{i=1}^{\infty}\left\{|S_{i}|<\beta_{i}m\right\}.\end{aligned}$$

Let ${\bar{S}_{i}}=\hat{M}(t)\setminus S_{i}$ and $S_0=\hat{M}(t).$ Then by the definitions of $\bar{S}_{i}$ and $S_{i},\bar{\bar{S}}_{i-1}\subseteq\bar{S}_{i}$ for all $i>0.$ Furthermore, note that $\operatorname*{lim}_{i\to\infty}m_{i}=0.$ So there must exist an integer $j$ such that $\bar{S}_{i}=  M(t)$ for all $i>j$, and $\hat{M}(t)=\bigcup_{i=1}^{\infty}(\bar{S}_{i}\setminus\bar{\bar{S}}_{i-1}).$ Finally, by the definition of $\bar{S}_{i},T_{t-1}(e)>m_{i}$ for all $e\in\bar{S}_{i}.$ Now suppose that event $\bar{G}$ happens. Then:

$$\begin{aligned}\sum_{e\in\hat{M}(t)}\frac{1}{\sqrt{T_{t-1}(e)}}<\sum_{i=1}^{\infty}\sum_{e\in\bar{S}_{i}\setminus\bar{S}_{i-1}}\frac{1}{\sqrt{m_{i}}}=\sum_{i=1}^{\infty}\frac{|\bar{S}_{i}\setminus\bar{S}_{i-1}|}{\sqrt{m_{i}}}\:,\end{aligned}$$

We apply the basic properties of $S_i,\bar{S}_i$ and $m_i$ we can know:

$$\begin{aligned}\sum_{i=1}^{\infty}|\bar{S}_{i}\setminus\bar{S}_{i-1}|\frac{1}{\sqrt{m_{i}}}&=\sum_{i=1}^{\infty}(|S_{i-1}\setminus S_{i}|)\frac{1}{\sqrt{m_{i}}}=\sum_{i=1}^{\infty}(|S_{i-1}|-|S_{i}|)\frac{1}{\sqrt{m_{i}}}\\&=\frac{|S_{0}|}{\sqrt{m_{1}}}+\sum_{i=1}^{\infty}|S_{i}|\left(\frac{1}{\sqrt{m_{i+1}}}-\frac{1}{\sqrt{m_{i}}}\right)\\&<\frac{\beta_{0}K}{\sqrt{m_{1}}}+\sum_{i=1}^{\infty}\beta_{i}m\left(\frac{1}{\sqrt{m_{i+1}}}-\frac{1}{\sqrt{m_{i}}}\right)=\sum_{i=1}^{\infty}(\beta_{i-1}-\beta_{i})m\frac{1}{\sqrt{m_{i}}}\:.\end{aligned}$$

The first two equalities follow from the definitions of $\bar{S}_i$ and $S_i.$ The inequality follows from the facts that $|S_i|<\beta_im$ for
all $i>0$ and $|S_0|\leq\beta_0m.$ 
In addition, let event $\mathcal{F}_t$ happen. Then:

$$\begin{aligned}\Delta_{  M(t)}\leq2\sum_{e\in\hat{M}(t)}\sqrt{\frac{2\log n}{T_{t-1}(e)}}<2\sqrt{2\log n}\sum_{i=1}^{\infty}\frac{(\beta_{i-1}-\beta_{i})m}{\sqrt{m_{i}}}=2\sqrt{2}\Delta_{M(t)}\sum_{i=1}^{\infty}\frac{\beta_{i-1}-\beta_{i}}{\sqrt{\alpha_{i}}}\leq\Delta_{  M(t)}\end{aligned}$$

where the last inequality is due to our assumption in \ref{Eq:alphabetaineq} The above is clearly a contradiction. As a result, $\bar{G}$ cannot happen, and event $G_i$ must happen for some $i$.

\subsection{Proof of Lemma \ref{Le:fineq}}

First, we define distribution $\mathcal{B}$ as if $X\sim B(p)$ then $X= 1$ with probability $p$ and $X=0$ with probability $1-p$ Then we construct combinatorial model instance $\nu= ( \nu_{1}, \nu_{2})$ and two combinatorial bandit instance $\nu_{1}= ( B(\frac{1}{2}-\zeta) , B(\frac{1}{2}))$ and $\nu _{2}= ( B(\frac{1}{2}-\zeta) , B(\frac{1}{2}+2g(t))$.(we can ignore numbers of basic action $d$ as we can construct $(X_1,X_2)=\nu_i$ with $i=1,2$ and $(X_3,...,X_d)=\vec{0}$ therefore for every super arm $M\in\mathcal{M}$ we can only observe feedback from $X_1,X_2$)

Without loss of generality we can assume $\zeta\in[0,1)$ and $g(t)\in[0,\frac{1}{8}]$. Then we have the estimation error $\Delta_\mu$ being $\Delta_{\nu_1}=\zeta$ and $\Delta_{\nu_{2}}=\zeta+2g(t)$ in two cases. We define the minimum distance test $\psi(\widehat{\Delta}_{\mu})$ that is associated to $\widehat{\Delta}_{\mu}$ by $$\psi(\widehat\Delta_\mu)=\arg\min_{i=1,2}|\widehat\Delta_\mu-\Delta_{\nu_i}|.$$

If $\psi(\widehat{\Delta}_{\mu})=1$, there is obviously $|\widehat{\Delta}_{\mu}-\Delta_{\nu_{1}}|\leq|\widehat{\Delta}_{\mu}-\Delta_{\nu_{2}}|.$ By the triangle inequality, we know that, if
$\psi(\hat{\Delta}_{\mu})=1$,
$$|\widehat{\Delta}_\mu-\Delta_{\nu_2}|\geq|\Delta_{\nu_1}-\Delta_{\nu_2}|-|\widehat{\Delta}_\mu-\Delta_{\nu_1}|\geq|\Delta_{\nu_1}-\Delta_{\nu_2}|-|\widehat{\Delta}_\mu-\Delta_{\nu_2}|,$$
which directly implies that
$$|\widehat{\Delta}_{\mu}-\Delta_{\nu_{2}}|\geq\frac{1}{2}|\Delta_{\nu_{1}}-\Delta_{\nu_{2}}|=g(t).$$
Symmetrically, if $\psi(\widehat{\Delta}_\mu)=2$, we wiil see
$$|\widehat{\Delta}_\mu-\Delta_{\nu_1}|\geq\frac{1}{2}|\Delta_{\nu_1}-\Delta_{\nu_2}|=g(t).$$
Therefore, we can apply the above to show
$$\begin{aligned}\operatorname*{inf}_{\hat{\Delta}_{\mu}}\operatorname*{max}_{\nu\in\mathcal{V}_0}\mathbb{P}_{\nu}\left(|\widehat{\Delta}_{\mu}-\Delta_{\nu}|\geq g(t)\right)&\geq\operatorname*{inf}_{\hat{\Delta}_{\mu}}\operatorname*{max}_{i\in\{1,2\}}\mathbb{P}_{\nu_{i}}\left(|\widehat{\Delta}_{\mu}-\Delta_{\nu_{i}}|\geq g(t)\right)\\&\geq\operatorname*{inf}_{\hat{\Delta}_{\mu}}\operatorname*{max}_{i\in\{1,2\}}\mathbb{P}_{\nu_{i}}\left(\psi(\widehat{\Delta}_{\mu})\neq i\right)\\&\geq\operatorname*{inf}_{\psi}\operatorname*{max}_{i\in\{1,2\}}\mathbb{P}_{\nu_{i}}\left(\psi\neq i\right).\end{aligned}$$  

where the last infimum is taken over all tests $\psi$ based on $\mathcal{H}_t$ that take values in $\{1,2\}.$

$$\begin{aligned}\operatorname*{inf}_{\hat{\Delta}_{\mu}}\operatorname*{max}_{\nu\in\mathcal{V}_0}\mathbb{P}_{\nu}\left(|\widehat{\Delta}_{\mu}-\Delta_{\nu}|\geq g(t)\right)&\geq\frac{1}{2}\inf_{\psi}\left(\mathbb{P}_{\nu_{1}}(\psi=2)+\mathbb{P}_{\nu_{2}}(\psi=1)\right)=\frac{1}{2}\left[1-\mathrm{TV}(\mathbb{P}_{\nu_{1}},\mathbb{P}_{\nu_{2}})\right]\\&\geq\frac{1}{2}\left[1-\sqrt{\frac{1}{2}\mathrm{KL}(\mathbb{P}_{\nu_{1}},\mathbb{P}_{\nu_{2}})}\right]\geq\frac{1}{2}\left[1-\sqrt{\frac{16g(t)^{2}}{3\Delta_{\nu_{1}}}\mathcal{R}_{\nu_1}\big(n,\pi\big)}\right].\end{aligned}$$

where the equality holds due to Neyman-Pearson lemma and the second inequality holds due to Pinsker's inequality, and the third inequality holds due to the following:

$$\begin{aligned}\mathrm{KL}(\mathbb{P}_{\nu_{1}},\mathbb{P}_{\nu_{2}})&=\sum_{t=1}^{n}\mathbb{E}_{\nu_{1}}[\mathrm{KL}(P_{1,A_{t}},P_{2,A_{t}})]=\sum_{i=1}^{2}\mathbb{E}_{\nu_{1}}[T_{i}(n)]\mathrm{KL}(P_{1,i},P_{2,i})\\&=\mathrm{KL}(B(\frac{1}{2}),B(\frac{1}{2}+2g(t)))\left(\mathbb{E}_{\nu_{1}}[T_{2}(n)]\right)\leq\frac{32(g(t))^{2}}{3\Delta_{\nu_{1}}}\mathcal{R}_{\nu_1}\big(n,\pi\big).\end{aligned}$$

where we use
$$\begin{aligned}\operatorname{KL}(B(\frac{1}{2}),B(\frac{1}{2}+2g(t)))
&=\frac{1}{2}\cdot\log\frac{\frac{1}{2}}{\frac{1}{2}+2g(t)}+\frac{1}{2}\cdot\log\frac{\frac{1}{2}}{\frac{1}{2}-2g(t)}
=\frac{1}{2}\cdot\log\left(\frac{1}{1-16g(t)^2}\right)
\\&\leq\frac{1}{2}(\frac{1}{1-16g(t)^2}-1)\leq\frac{8g(t)^2}{1-16g(t)^2}\leq\frac{32}{3}g(t)^2.
\end{aligned}$$

As we know the history $\mathcal{H}_t$ is generated by selection rule $\pi$ and $\Delta_{\nu_1}\mathbb{E}_{\nu_1}[T_2(n)]$ is basically the expected regret of $\nu_1$, which is just the definition of regret. It shows that the last inequality holds and thus we finish our proof.

\section{Construction of Restricted Bandit Structure}
\label{app:rst_bandit}

We consider the $2d_0$-th ($d_0\in\mathbb{N^*}$) armed fully combinatorial bandit with super-arm set as 
$$ \mathcal{M}=\{\{1\},\{2\},\{3,4\},\{5,6\},....,\{2d_0-1,2d_0\}\}$$
and we can see that for every basic action $e\in \{3,4,...,2d_0\}\subset \mathcal{A}$, we can see that $\boldsymbol{\theta}_{e}\notin\mathrm{span}\{\boldsymbol{\theta}_{M}:M\in\mathcal{M}\}$, that means we cannot express the reward of basic action $\boldsymbol{\theta}_{e}w_t$ within the linear space drawn by observable reward $\{\boldsymbol{\theta}_{M}w_t:M\in\mathcal{M}\}$. This shows that only the reward of basic action $\{1\},\{2\}$ is estimable.

\section{Existing Techincal Results}
In this section we state some well-known theorems used in our paper:

\begin{theorem}{(Bernstein's Inequality).} Let $X_1, X_2, \cdots$ be a martingale difference sequence, such that $| X_{t}| \leq \alpha _{t}$ for $a$ non-decreasing deterministic sequence $\alpha _{1}, \alpha _{2}, \cdots$ with probability 1. Let $M_{t}: = \sum _{\tau = 1}^{t}X_{\tau }$ $be$ $martingale.$ Let $\bar{V} _{1}, \bar{V} _{2}, \cdots$ $be$ $a$ deterministic upper bounds $on$ the variance $V_{t}: = \sum _{\tau = 1}^{t}\mathbb{E} [ X_{\tau }^{2}| X_{1}, \cdots , X_{\tau - 1}]$ $of$ the martingale $M_{t}$, such that $\bar{V} _{t}- s$ satisfy $\sqrt {\frac {\ln ( \frac 2\delta ) }{( e- 2) V_{t}}\leq \frac 1{\alpha _{t}}}$. Then,  with probability greater than $1-\delta$, for all $t$:
$$|M_t|\leq2\sqrt{(e-2)\bar{V}_t\ln\frac{2}{\delta}}.$$
\end{theorem}

\begin{theorem}{(Neyman-Pearson Lemma).}
Let $\mathbb{P} _0$ and $\mathbb{P} _1$ $be$ two probability measures.
Then for any test $\psi$, it holds
$$\mathbb{P}_0(\psi=1)+\mathbb{P}_1(\psi=0)\geq\int\min(p_0,p_1).$$
Moreover, the equality holds for the Likelihood Ratio test $\psi ^{\star }= \mathbb{I} ( p_{1}\geq p_{0}) .$
\end{theorem}
An direct corollary is that:
\begin{corollary}
$$\inf_\psi\left[\mathbb{P}_0(\psi=1)+\mathbb{P}_1(\psi=0)\right]=1-TV(\mathbb{P}_0,\mathbb{P}_1).$$   
\end{corollary}
\begin{proof}
Denote that $\mathbb{P}_0$ and $\mathbb{P}_1$ are defined on the probability space $(\mathcal{X},\mathcal{A}).$ By the definition of the total variation distance, we have
$$\begin{aligned}TV(\mathbb{P}_{0},\mathbb{P}_{1})&=\sup_{R\in\mathcal{A}}|\mathbb{P}_{0}(R)-\mathbb{P}_{1}(R)|\\&=\sup_{R\in\mathcal{A}}\left|\int_{R}p_{0}-p_{1}\right|\\&=\frac{1}{2}\int|p_{0}-p_{1}|\\&=1-\int\min(p_{0},p_{1})\\&=1-\inf_{\psi}\left[\mathbb{P}_{0}(\psi=1)+\mathbb{P}_{1}(\psi=0)\right].\end{aligned}$$
where the last equality applies the Neyman-Pearson Lemma, and the fourth equality holds due to the fact that
$$\begin{aligned}\int|p_{0}-p_{1}|&=\int_{p_{1}\geq p_{0}}(p_{1}-p_{0})+\int_{p_{1}<p_{0}}(p_{0}-p_{1})\\&=\int_{p_{1}\geq p_{0}}p_{1}+\int_{p_{1}<p_{0}}p_{0}-\int\min(p_{0},p_{1})\\&=1-\int_{p_{1}<p_{0}}p_{1}+1-\int_{p_{1}\geq p_{0}}p_{0}-\int\min(p_{0},p_{1})\\&=2-2\int\min(p_{0},p_{1}).\end{aligned}$$  
And we finished our proof.
\end{proof}

\begin{theorem}{(Pinsker’s inequality).}
Let $\mathbb{P}_1$ and $\mathbb{P} _2$ be two probability measures such that $\mathbb{P} _1\ll \mathbb{P} _2$. Then,
$$TV(\mathbb{P}_1,\mathbb{P}_2)\leq\sqrt{\frac{1}{2}KL(\mathbb{P}_1,\mathbb{P}_2)}.$$    
\end{theorem}

\end{onecolumn}

\end{document}